\documentclass{article}

\usepackage{booktabs}
\usepackage{multirow}
\usepackage{array}
\usepackage{booktabs}
\usepackage{tabularx}
\usepackage{graphicx}

\usepackage{hyperref}

 \usepackage[accepted]{icml2025}

\usepackage[T1]{fontenc}
\usepackage[latin9]{inputenc}
\usepackage{refstyle}
\usepackage{mathrsfs}
\usepackage{amsmath}
\usepackage{amsthm}
\usepackage{amssymb}
\usepackage{bm}

\usepackage{graphicx}
\usepackage{caption}
\usepackage{subcaption}
\usepackage{tabularx}
\usepackage{xcolor}
\usepackage{subcaption}
\usepackage{rotating}
\usepackage{dblfloatfix}    

\usepackage{booktabs}
\usepackage{longtable}
\usepackage{xurl}
\usepackage{array}
\usepackage{multirow}
\newcolumntype{C}[1]{>{\centering\arraybackslash$}m{#1}<{$}}
\usepackage{tablefootnote}
\usepackage{threeparttable}

\usepackage{mathtools}
\usepackage{thmtools}

\theoremstyle{plain}
\newtheorem{theorem}{Theorem}[section]

\newtheorem{lemma}[theorem]{Lemma}

\theoremstyle{definition}

\theoremstyle{remark}

\usepackage[textsize=tiny]{todonotes}

\begin{document}

\twocolumn[
\icmltitlerunning{SDP-CROWN: Efficient Bound Propagation for Neural Network Verification with Tightness of Semidefinite Programming}
\icmltitle{SDP-CROWN: Efficient Bound Propagation for Neural Network Verification with Tightness of Semidefinite Programming}



\icmlsetsymbol{equal}{*}

\begin{icmlauthorlist}
\icmlauthor{Hong-Ming Chiu}{yyy}
\icmlauthor{Hao Chen}{yyy}
\icmlauthor{Huan Zhang}{yyy}
\icmlauthor{Richard Y. Zhang}{yyy}
\end{icmlauthorlist}

\icmlaffiliation{yyy}{Department of Electrical and Computer Engineering, University of Illinois at Urbana-Champaign}

\icmlcorrespondingauthor{Hong-Ming Chiu}{hmchiu2@illinois.edu}
\icmlcorrespondingauthor{Hao Chen}{haoc8@illinois.edu}
\icmlcorrespondingauthor{Huan Zhang}{huan@huan-zhang.com}
\icmlcorrespondingauthor{Richard Y. Zhang}{ryz@illinois.edu}

\icmlkeywords{Machine Learning, Neural Network Verification, ICML}

\vskip 0.3in
]



\printAffiliationsAndNotice{}  

\global\long\def\R{\mathbb{R}}%
\global\long\def\S{\mathbb{S}}%
\global\long\def\AA{\mathcal{A}}%
\global\long\def\XX{\mathcal{X}}%
\global\long\def\BB{\mathcal{B}}%
\global\long\def\EE{\mathcal{E}}%
\global\long\def\LL{\mathcal{L}}%
\global\long\def\L{\mathscr{L}}%
\global\long\def\one{\mathbf{1}}%
\global\long\def\gap{\mathrm{gap}}%
\global\long\def\feas{\mathrm{feas}}%
\global\long\def\inner#1#2{\left\langle #1,#2\right\rangle }%

\global\long\def\ub{\mathrm{ub}}%
\global\long\def\lb{\mathrm{lb}}%

\global\long\def\1{\mathbf{1}}%
\global\long\def\x{\mathbf{x}}%
\global\long\def\e{\mathbf{e}}%
\global\long\def\u{\mathbf{u}}%
\global\long\def\v{\mathbf{v}}%
\global\long\def\y{\mathbf{y}}%
\global\long\def\z{\mathbf{z}}%
\global\long\def\f{\mathbf{f}}%
\global\long\def\F{\mathcal{F}}%
\global\long\def\bvec{\mathbf{h}}%
\global\long\def\Wmat{\mathbf{W}}%
\global\long\def\X{\mathbf{X}}%
\global\long\def\U{\mathbf{U}}%
\global\long\def\G{\mathbf{G}}%
\global\long\def\b#1{\mathbf{#1}}%
\global\long\def\relu{\mathrm{ReLU}}%
\global\long\def\diag{\mathrm{diag}}%
\global\long\def\Diag{\mathrm{Diag}}%
\global\long\def\sumsm{{\textstyle \sum}}%
\global\long\def\I{\mathcal{I}}%
\global\long\def\E{\mathcal{E}}%

\global\long\def\rank{\operatorname{rank}}%

\global\long\def\dist{\operatorname{dist}}%

\global\long\def\diag{\operatorname{diag}}%

\global\long\def\tr{\operatorname{tr}}%

\global\long\def\relu{\operatorname{ReLU}}%


\global\long\def\forall{\mbox{for all }}%


\begin{abstract}
Neural network verifiers based on linear bound propagation scale impressively
to massive models but can be surprisingly loose when neuron coupling
is crucial. Conversely, semidefinite programming (SDP) verifiers capture
inter-neuron coupling naturally, but their cubic complexity restricts them to only small models. In this paper, we propose SDP-CROWN,
a novel hybrid verification framework that combines the tightness
of SDP relaxations with the scalability of bound-propagation verifiers.
At the core of SDP-CROWN is a new linear bound---derived via SDP
principles---that explicitly captures $\ell_{2}$-norm-based inter-neuron
coupling while adding only one extra parameter per layer. This bound
can be integrated seamlessly into any linear bound-propagation pipeline,
preserving the inherent scalability of such methods yet significantly
improving tightness. In theory, we prove that our inter-neuron bound
can be up to a factor of $\sqrt{n}$ tighter than traditional per-neuron
bounds. In practice, when incorporated into the state-of-the-art $\alpha$-CROWN
verifier, we observe markedly improved verification performance on
large models with up to 65 thousand neurons and 2.47 million parameters,
achieving tightness that approaches that of costly SDP-based methods.
\end{abstract}

\section{Introduction}

\begin{figure}[!h]
    \centering
    \includegraphics[width=1.0\linewidth]{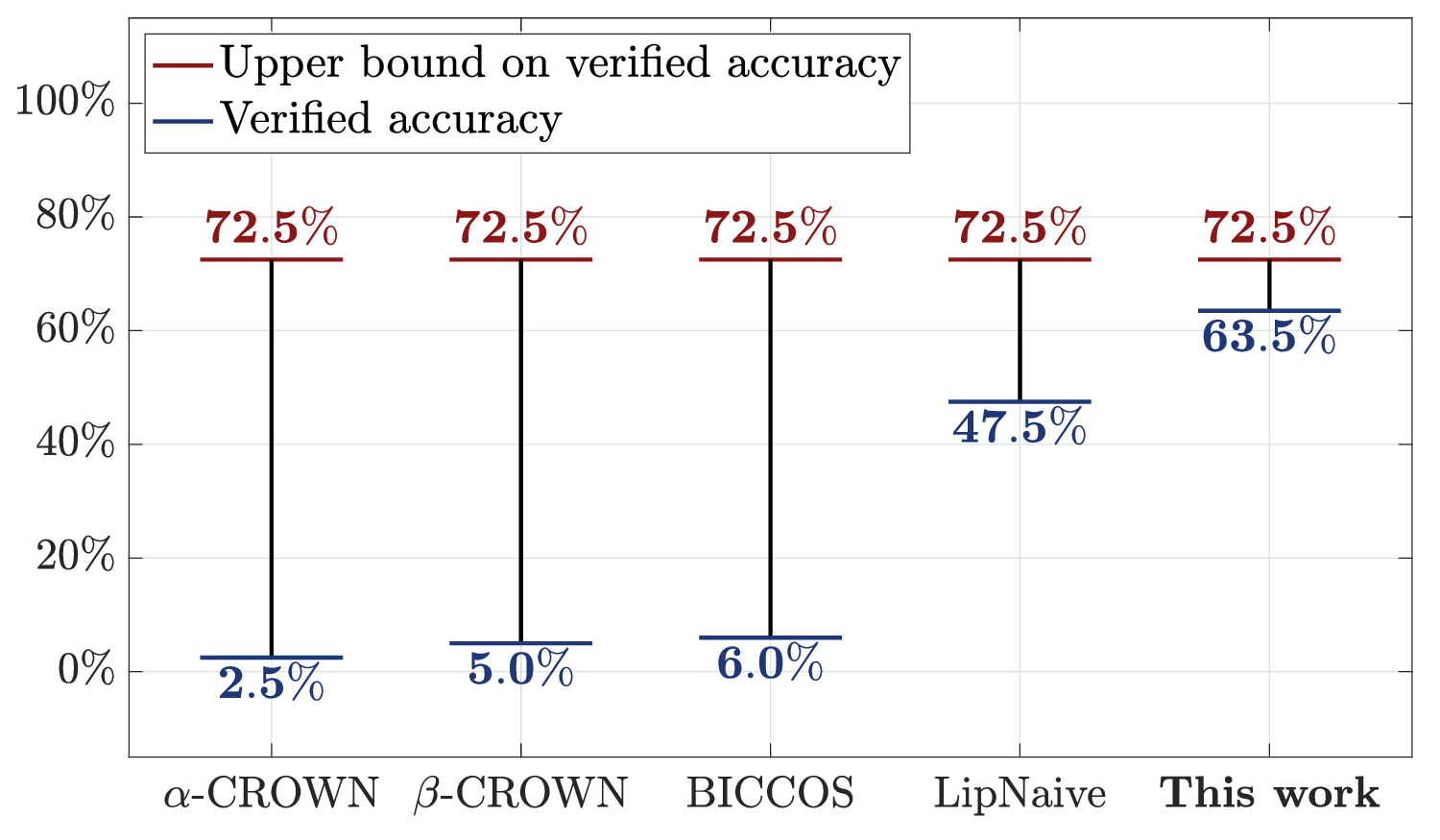}
    \caption{\textbf{Verification of the ConvLarge network on the CIFAR-10 dataset (six convolutional layers + three fully connected layers, $\approx$2.47M parameters
and 65k neurons) under $\ell_{2}$ adversaries.} State-of-the-art
bound propagation algorithms $\alpha$-CROWN, $\beta$-CROWN and BICCOS 
yield surprisingly loose relaxations  under $\ell_2$ adversaries (verified accuracy 2.5\%, 5.0\% and
6.0\%, respectively) at high cost (up to 289s per example). A naive
Lipschitz baseline, which multiples per-layer $\ell_2$ Lipschitz constants directly (``LipNaive''), outperforms them (47.5\% verified
accuracy) with negligible runtime. In contrast, our proposed method 
achieves a striking 63.5\% verified accuracy while keeping runtime
moderate (73s). See Section~\ref{sec:verified accuracy} for experimental details; note that the model is too large to be verified with traditional SDP methods.}\label{fig:verified_accuracy}
\vspace{-0.5em}
\end{figure}

Neural network verification is critical for ensuring that models deployed
in safety-critical applications adhere to robustness and safety requirements.
Among various verification methods, linear bound propagation based approaches \cite{zhang2018crown,wang2018efficient,wong2018provable,dvijotham2018dual,singh2018fast,singh2019abstract,xu2020automatic} have emerged
as the dominant approach due to their effectiveness and scalability.
The core idea is to construct linear functions that provide pointwise
upper and lower bounds for each nonlinear activation function and
recursively propagate these bounds through the network. This method
has proven particularly effective in certifying robustness against
$\ell_{\infty}$-norm perturbations, where individual input features
are perturbed within fixed limits. Notably, many highly ranked verifiers in VNN-COMP verification competition rely on bound propagation due
to its success in scaling to large networks \cite{brix2023first,brix2024fifth}.  

Despite this success, bound propagation performs surprisingly poorly
under $\ell_{2}$-norm perturbations, as shown in Figure~\ref{fig:verified_accuracy}.
Unlike $\ell_{\infty}$-norm perturbations, which treat each neuron
independently, $\ell_{2}$-norm perturbations impose inter-neuron coupling. This coupling introduces dependencies
between neurons that bound propagation, designed to handle features
individually, cannot effectively capture. As a result, it often produces
loose and overly conservative output bounds. The $\ell_{2}$-norm
setting is critical not only as a benchmark for evaluating neuron
coupling \cite{Szegedy2014intriguing} but also for verifying real-world adversarial examples, such
as semantic perturbations, which are commonly modeled using $\ell_{2}$-norm
perturbations applied through generative layers \cite{wong2021learning,barrett2022certifiably}.

To overcome this limitation, semidefinite programming (SDP) methods \cite{raghunathan2018semidefinite,dathathri2020enabling,fazlyab2020safety,anderson2021partition,newton2021exploiting,chiu2023tight}
have been developed to explicitly model inter-neuron dependencies.
These methods optimize over a dense $n\times n$ coupling matrix,
yielding significantly tighter bounds compared to bound propagation.
However, their cubic $O(n^{3})$ time complexity restricts their application
to relatively small models, and renders them impractical for realistic
neural networks.

To bridge the gap between scalability and tightness, we introduce
SDP-CROWN, a hybrid verification framework that combines the tightness
of SDP with the efficiency of bound propagation. At the core of our
framework is a linear bound derived through SDP principles that efficiently
incorporates $\ell_{2}$-norm-based inter-neuron dependencies. A key
feature of our bound is that it introduces only a single new parameter
per layer, in contrast to traditional SDP methods, which require $n^{2}$
parameters per layer of $n$ neurons. As a result, our bound preserves
the scalability of existing bound propagation methods. It can be seamlessly
integrated into existing linear bound propagation verifiers such as CROWN and $\alpha$-CROWN, hence significantly
tightening it for $\ell_{2}$ perturbations.

In theory, we prove that our proposed inter-neuron bound can be up
to $\sqrt{n}$ times tighter than the per-neuron bounds commonly used
in scalable verifiers. In practice, when incorporated into the
$\alpha$-CROWN verifier, we find that SDP-CROWN significantly improves
verification performance, achieving bounds close to those of expensive
SDP-based methods while scaling to models containing over 65 thousand
neurons and 2.47 million parameters. Our extensive experiments demonstrate
that SDP-CROWN consistently enhances $\ell_2$ robustness certification rates
across various architectures without compromising computational efficiency,
making it well-suited for large-scale models where traditional SDP
methods fail to scale.

\subsection{Related work}
To the best of our knowledge, this work is the first to apply SDP relaxations to efficiently tighten the linear bound propagation under $\ell_2$-norm perturbations.

SDP relaxation is the preferred approach for neural network verification against $\ell_2$-norm perturbations. Due to its ability to capture second-order information, SDP-based methods provide tight verification under $\ell_2$ adversaries \citep{chiu2023tight}. Several extensions have been proposed to further tighten the SDP relaxation by introducing linear cuts \cite{batten2021efficient} and nonconvex cuts \cite{ma2020strengthened}, and to accommodate general activation functions \cite{fazlyab2020safety}. However, SDP relaxation does not scale to medium-to-large scale models. Even with state-of-the-art SDP solvers and hardware acceleration \cite{dathathri2020enabling,chiu2023tight}, they remain computationally prohibitive for models containing more than 10 thousand neurons.

In addition to SDP relaxation, bound propagation methods \cite{zhang2018crown,singh2018fast,singh2019abstract,wang2018efficient,dvijotham2018dual,hashemi2021osip,xu2020automatic,xu2021fast} for certifying $\ell_2$ adversaries can be tightened using a branch-and-bound procedure \cite{wang2021beta,de2021scaling,ferrari2022complete,shi2024genbab} by splitting unstable ReLU neurons into two subdomains, or by introducing nonlinear cutting planes \cite{zhang2022general,zhou2024scalable} to capture the shape of $\ell_2$ adversaries. These methods have proven effective when the $\ell_2$-norm perturbation is small, as there are relatively fewer unstable neurons. However, they can become ineffective for larger perturbations, and can completely fail, as demonstrated in Figure~\ref{fig:verified_accuracy}.

Alternatively, $\ell_2$ adversaries can be certified by lower bounding the robustness margin (\ref{eq:attack}) using the network's global Lipschitz constant. To estimate this constant, \cite{fazlyab2019efficient} solve the Lipschitz constant estimation problem using SDP relaxations;  \cite{huang2021training,leino21gloro,hu2023unlocking} incorporate a Lipschitz upper bound during training; and \cite{li2019preventing,trockman2021orthogonalizing,singla2022improved,meunier2022dynamical,xu2022lot,araujo2023unified} design neural network architectures that are provably 1-Lipschitz. While these methods perform well on networks with a small global Lipschitz constant, verifying robustness based solely on the Lipschitz constant can still be overly conservative, as illustrated in Figure~\ref{fig:verified_accuracy}.

\subsection{Notations}
We use the subscript $x_i$ to denote indexing. We use $\1$ to denote a column of ones. We use $\odot$ to denote elementwise multiplication. We use $e_i$ to denote $i$-th  standard basis vector. We use $X\succeq 0$ to denote $X$ being positive semidefinite. We use $|\cdot|$ to denote the \emph{elementwise} absolute value, and $\|\cdot\|_p$ to denote the vector $\ell_p$ norm.

\section{Preliminaries}

\subsection{Problem description}
Consider the task of classifying a data point $x\in\R^{n}$ as belonging to
the $i$-th of $q$ classes using a $N$-layer feedforward neural
network $f:\R^{n}\to\R^{q}$. The network aims to generate a prediction
vector that takes on its maximum value at the $i$-th element, i.e.,
$e_{i}^{T}f(x)>e_{j}^{T}f(x)$ for all incorrect labels $j\ne i$.
We define the neural network $f(x)=z^{(N)}$ recursively as 
\begin{equation}\label{eq:model}
x^{(k)}=\relu(z^{(k)}),\ \ z^{(k)}=W^{(k)}x^{(k-1)},\ \ x^{(0)}=x
\end{equation}
for $k\in\{1,2,\dots,N\}$ where $\relu(x)=\max\{x,0\}$ and $W^{(k)}$
denote weight matrices. Without loss of generality, we ignore biases.

Given an input $\hat{x}$ of truth class $i$, the problem of verifying
the neural network $f$ to have no adversarial example $x\approx\hat{x}$
mislabeled as the incorrect class $j\ne i$ can be posed as: 
\begin{equation}\label{eq:attack}
d_{j}=\min_{x\in\XX}\ c^{T}f(x)\quad\text{s.t.}\quad(\ref{eq:model}),
\end{equation}
where $c=e_{i}-e_{j}$ and $\XX$ is a convex input set that models
the adversarial perturbations. Popular choices include the elementwise
bound 
\[
\BB_{\infty}(\hat{x},\hat{\rho})=\{x\mid|x_{i}-\hat{x}_{i}|\leq\hat{\rho}_{i}\mbox{ for all }i\}
\]
and the $\ell_{2}$ norm ball 
\[
\BB_{2}(\hat{x},\rho)=\{x\mid\|x-\hat{x}\|_{2}\leq\rho\},
\]
where $\hat{x}\in\R^{n}$ is a center point, and $\hat{\rho}\in\R^{n}$
and $\rho\in\R$ are the radii. The resulting vector $d\in\R^{q}$
is a \emph{robustness margin} against misclassification. If $d\ge0$
over all of its elements, then there exists no adversarial example
$x$ within a distance of $\rho$ that can be misclassified.

\subsection{Semidefinite programming (SDP) relaxation}
Semidefinite relaxation is a convex relaxation method to compute lower bounds on (\ref{eq:attack}).
In this work, we focus our attention on the SDP relaxation used in \citet{brown2022unified}
that utilizes the positive/negative splitting of the preactivations
$u_{i}=x_{i}$, $v_{i}=x_{i}-z_{i}$ to rewrite the equality constraints
$x_{i}=\relu(z_{i})$ as 
\begin{gather*}
x_{i}=u_{i},\quad u_{i}v_{i}=0,\quad u_{i}\geq0,\quad v_{i}\geq0.
\end{gather*}
Adding $[1\ u_{i}\ v_{i}]^{T}[1\ u_{i}\ v_{i}]\succeq0$ and relaxing $U_{i}=u_{i}^{2}$
and $V_{i}=v_{i}^{2}$ yields the SDP relaxation of the ReLU activation
\begin{equation}
x_{i}=u_{i},\ \ u_{i}\geq0,\ \ v_{i}\geq0,\ \ \begin{bmatrix}1 & u_{i} & v_{i}\\
u_{i} & U_{i} & 0\\
v_{i} & 0 & V_{i}
\end{bmatrix}\succeq0.\label{eq:sdp_relu}
\end{equation}
Similarly, the SDP relaxation of $\BB_{2}(\hat{x},\rho)$ is given
by 
\begin{equation}\label{eq:sdp_l2}
\begin{gathered}
\sum_{i=1}^{n}U_{i}-2(u_{i}-v_{i})\hat{x}_{i}+V_{i}+\hat{x}_{i}^{2}\leq\rho^{2},\\
u_{i}\geq 0,\quad v_{i}\geq 0,\quad
\begin{bmatrix}1 & u_{i} & v_{i}\\
u_{i} & U_{i} & 0\\
v_{i} & 0 & V_{i}
\end{bmatrix}\succeq0.
\end{gathered}
\end{equation}
The SDP relaxation of (\ref{eq:attack}) can be derived via (\ref{eq:sdp_relu})
and (\ref{eq:sdp_l2}). While SDP relaxations are typically tighter than most other convex relaxation methods, existing approaches solve the SDP relaxation via interior point
method \citep{brown2022unified} or low-rank factorization method
\citep{chiu2023tight}. Those methods incur approximately cubic time
complexity and are not scalable to medium-scale models.

\section{Looseness of bound propagation for $\ell_{2}$-norm perturbations}

Linear bound propagation is one of the state-of-the-art approaches for finding upper
and lower bounds on (\ref{eq:attack}). In this section, we explain
why the approach can be unusually loose when faced with an $\ell_{2}$
perturbation set like $\XX=\BB_{2}(\hat{x},\rho)$, which is the classic
example when inter-neuron coupling strongly manifests. For simplicity,
we focus on finding a lower bound for (\ref{eq:attack}). 

At a high level, all bound propagation methods solve (\ref{eq:attack})
by defining a set of linear relaxations $x\mapsto g^{T}x+h$ that \emph{pointwise}
lower bound the original function $c^{T}f(x)$ across the input set
$\XX$, as in 
\[
\L(\XX)=\{(g,h)\mid c^{T}f(x)\ge g^{T}x+h\ \forall x\in\XX\}.
\]
The linear relaxation corresponding to each $(g,h)\in\L(\XX)$ can
be minimized to yield a valid lower bound on the original problem
(\ref{eq:attack}). This bound can be further tightened by optimizing
over the linear relaxations themselves:
\[
\min_{x\in\XX}\ c^{T}f(x)\ge\max_{(g,h)\in\L(\XX)}\min_{x\in\XX}\ g^{T}x+h.
\]
In fact, one can show by a duality argument that the bound above
is in fact exactly tight, i.e. the inequality holds with equality.
Unfortunately, the set of linear relaxations $\L(\XX)$ is also intractable
to work with.

Instead, all bound propagation methods work by constructing parameterized
families of linear relaxations $x\mapsto g(\alpha)^{T}x+h(\alpha)$
for $0\le\alpha\le1$ that provably satisfy $(g(\alpha),h(\alpha))\in\L(\XX)$.
The tightest bound on (\ref{eq:attack}) that could be obtained from
the family of relaxations then reads
\begin{equation}
\min_{x\in\XX}\ c^{T}f(x)\ge\max_{0\le\alpha\le1}\min_{x\in\XX}\ g(\alpha)^{T}x+h(\alpha).\label{eq:heuLP}
\end{equation}
Note that the inner minimization is a convex program that can be efficiently
evaluated for many common choices of $\XX$, such as the elementwise
bound or any $\ell_{p}$ norm ball. In practice, the parameter $\alpha$
can be maximized via projected gradient ascent or selected heuristically
as in \citet{zhang2018crown}. 

The tightness of the heuristic bound in (\ref{eq:heuLP}) is critically
driven by the quality of the parameterized relaxations $x\mapsto g(\alpha)^{T}x+h(\alpha)$.
The core insight of bound propagation methods is that a high-quality
choice of $g(\alpha),h(\alpha)$ satisfying the following
\[
c^{T}f(x)\ge g(\alpha)^{T}x+h(\alpha)\ \forall x\in\BB_{\infty}(\hat{x},\hat{\rho})
\]
can be constructed using the triangle relaxation of the ReLU activation,
alongside a forward-backward pass through the neural network; we refer
the reader to the Appendix~\ref{app:bound_propagation} for precise details. When the input
set is indeed an $\ell_{\infty}$-norm box $\XX=\BB_{\infty}(\hat{x},\hat{\rho})$,
\citet{salman2019convex} showed that this choice of $(g(\alpha),h(\alpha))\in\L(\BB_{\infty}(\hat{x},\hat{\rho}))$
is essentially \emph{optimal} per-neuron. This optimality provides a long-sought
explanation for the tightness of bound propagation under an $\ell_{\infty}$
adversary. 

However, when the input set $\XX$ is not an $\ell_{\infty}$-norm
box, bound propagation requires relaxing the input set $\XX\subseteq\BB_{\infty}(\hat{x},\hat{\rho})$
for the purposes of constructing $g(\alpha),h(\alpha)$. The resulting
bound on (\ref{eq:attack}) is valid by the following sequence of
inequalities
\begin{align}
\min_{x\in\XX}\ c^{T}f(x) & \ge\max_{(g,h)\in\L(\XX)}\min_{x\in\XX}\ g^{T}x+h\nonumber \\
 & \geq\max_{(g,h)\in\L(\BB_{\infty}(\hat{x},\hat{\rho}))}\min_{x\in\XX}\ g^{T}x+h\label{eq:box_relax}\\
 & \geq\max_{0\leq\alpha\leq1}\min_{x\in\XX}\ g(\alpha)^{T}x+h(\alpha).\nonumber
\end{align}
The problem is that a loose relaxation $\BB_{\infty}(\hat{x},\hat{\rho})\supseteq\XX$
causes a comparably loose relaxation $\L(\BB_{\infty}(\hat{x},\hat{\rho}))\subseteq\L(\XX)$
in (\ref{eq:box_relax}), hence introducing substantial conservatism
to the overall bound. 

The above explains the core mechanism for why bound propagation tends
to be loose for an $\ell_{2}$ adversary. The problem is that the tightest
$\ell_{\infty}$-norm box to fully contain a given $\ell_{2}$-norm
ball satisfies the following 
\[
\XX=\BB_{2}(\hat{x},\rho)\subseteq\BB_{\infty}(\hat{x},\one\rho).
\]
However, there are attacks in the box $x\in\{\pm\rho\}^{n}\subseteq\BB_{\infty}(\hat{x},\one\rho)$
with radii $\|x-\hat{x}\|_2=\sqrt{n}\rho$ that are a factor of $\sqrt{n}$ larger
than the radius $\rho$ of the original ball. Accordingly, relaxing
the $\ell_{2}$-norm ball into $\ell_{\infty}$-norm box can effectively
increase the attack radius by a factor of $\sqrt{n}$. Hence, the
resulting bounds on (\ref{eq:attack}) can also be a factor of $\sqrt{n}$
more conservative.

\newpage
\section{Proposed method}
Our core contribution in this paper is a high-quality family of linear
relaxations $x\mapsto g^{T}x+h(g,\lambda)$ for $g\in\R^n$ and $\lambda\ge0$ that provably satisfy the following
\[
c^{T}f(x)\ge g^{T}x+h(g,\lambda)\ \ \forall\ x\in\BB_{2}(\hat{x},\rho).
\]
Notice that our relaxation is constructed directly from the $\ell_{2}$-norm ball, i.e. $(g,h(g,\lambda))\in\L(\XX)$, which addresses the looseness of (\ref{eq:box_relax}) as we did not relax the $\ell_{2}$-norm ball into $\ell_{\infty}$-norm box. Due to space
constraints, we explain our construction only for the special case
of $f(x)\equiv\relu(x)$, while deferring the general case to the
Appendix~\ref{app:bound_propagation}.

One particle aspect of our construction is to take a linear relaxation from bound propagation $c^Tf(x) \geq g(\alpha)^{T}x+h(\alpha)$ for the box $\BB_{\infty}(\hat{x},\rho\1)\supseteq\BB_{2}(\hat{x},\rho)$, and then tightening the offset $h(g(\alpha),\lambda) \geq h(\alpha)$ while ensuring that it remains valid for the ball $\BB_{2}(\hat{x},\rho)$. In analogy with \citet{salman2019convex}, we prove in Section~\ref{sec:tightness}
that this choice of $h(g(\alpha),\lambda)$ is essentially optimal when $\hat x = 0$, and can therefore yield at most a factor of $\sqrt{n}$ reduction in conservatism for $\XX=\BB_{2}(\hat{x},\rho)$. At the same time, our new method
adds just one parameter $\lambda\ge0$ per layer, so it can be seamlessly
integrated into any bound propagation verifier with negligible overhead.
Integrating this technique into the $\alpha$-CROWN
verifier, we provide extensive computational verification in Section~\ref{sec:experiment}
showing that our method significantly improves verification
performance. The main theorem of our work is summarized below.

\begin{theorem}\label{thm:sdp_crown_lb} 
Given $c,\hat x\in\R^n$ and $\rho\geq 0$. The following holds
\[
c^{T}\relu(x)\geq g^{T}x+h(g,\lambda)\ \ \forall\ x\in\BB_{2}(\hat{x},\rho)
\]
for any $\lambda\geq0$ and $g\in\R^n$ where 
\[
h(g,\lambda)=-\frac{1}{2}\cdot\left(\lambda(\rho^2-\|\hat x\|_2^2)+\frac{1}{\lambda}\|\phi(g,\lambda)\|_2^2\right)
\]
and 
\[
\phi_i(g,\lambda) = \min\{c_i-g_i-\lambda\hat x_i,g_i+\lambda\hat x_i,0\}.
\]
\end{theorem}

Let us explain how Theorem~\ref{thm:sdp_crown_lb} can be used to lower bound
the attack problem (\ref{eq:attack}) in the special case of $f(x)\equiv\relu(x)$
and $\XX=\BB_{2}(\hat{x},\rho)$. First, we use the standard bound
propagation procedure to compute linear relaxations $x\mapsto g(\alpha)^{T}x+h(\alpha)$ that provably satisfy $(g(\alpha),h(\alpha))\in\L(\BB_{\infty}(\hat{x},\rho\one))$ for $0\leq\alpha\leq 1$. Then, we replace $h(\alpha)$ with the new choice $h(g(\alpha),\lambda)$ specified
in Theorem~\ref{thm:sdp_crown_lb} to ensure that $(g(\alpha),h(g(\alpha),\lambda))\in\L(\BB_{2}(\hat{x},\rho))$ for $\lambda\geq 0$. Both $\alpha$ and $\lambda$ can then be optimized to provide a tighter relaxation.
We note that the attack problem can also be lower bounded by directly optimizing over $g$ and $\lambda\geq 0$ (by treating $\alpha$ as unconstrained variables), and Theorem~\ref{thm:sdp_crown_lb} can be extended to handle more general input set $\XX$ such as an ellipsoid. We provide more details for these extensions in the Appendix~\ref{app:extension}.

In the remainder of this section, we provide a proof of Theorem \ref{thm:sdp_crown_lb}.
\subsection{Proof of Theorem~\ref{thm:sdp_crown_lb}}
Given any $c,g\in\R^n$ the process of finding the tightest possible $h$ such that $c^T\relu(x)\geq g^Tx+h$ holds within  within $\BB_2(\hat x,\rho)$ admits the following generic problem
\begin{equation}\label{eq:hopt}
\min_{x\in\R^n}\ c^T\relu(x)-g^Tx\quad\text{s.t.}\quad\|x-\hat x\|_2^2\leq \rho^2.
\end{equation}
Applying the positive/negative splitting $x=u-v$ where $u,v\geq 0$ and $u\odot v=0$ yields the following
\begin{equation*}
\begin{aligned}
\min_{u,v\in\R^n}\ & c^Tu-g^T(u-v)\\
\text{s.t. }\ &\|u\|_2^2 - 2(u-v)^T\hat x + \|v\|_2^2\leq \rho^2-\|\hat x\|_2^2\\
& u\geq 0,\quad v\geq 0,\quad u\odot v=0.
\end{aligned}
\end{equation*}
Though (\ref{eq:hopt}) is nonconvex due to the product of the two variables $u$ and $v$, a tight lower bound can be efficiently approximated via SDP relaxation described in (\ref{eq:sdp_relu}) and (\ref{eq:sdp_l2}). The SDP relaxation of (\ref{eq:hopt}) reads:
\begin{equation*}
	\begin{aligned}
		\min_{u,v,U,V\in\R^n}\ & c^Tu-g^T(u-v)\\
		\text{s.t. }\quad & (U+V)^T\1-2(u-v)^T\hat x \leq \rho^2-\|\hat x\|_2^2,\\
		&u\geq 0,\quad v\geq 0, \\
		& \begin{bmatrix}1&u_i&v_i\\u_i&U_i&0\\v_i&0&V_i\end{bmatrix}\succeq 0 \text{ for $i=1,\ldots,n$}.
	\end{aligned}
\end{equation*}
The SDP relaxation can be further simplified by applying Theorem 9.2 of \cite{vandenberghe2015chordal}:
\begin{equation*}
	\begin{aligned}
		\min_{\tilde u,\tilde v,u,v,U,V\in\R^n}\ & c^Tu-g^T(u-v)\\
		\text{s.t. }\quad\ & (U+V)^T\1-2(u-v)^T\hat x \leq \rho^2-\|\hat x\|_2^2,\\
		&u\geq 0,\quad v\geq 0,\quad \tilde u + \tilde v = 1,\\  &\begin{bmatrix}\tilde u_i&u_i\\u_i&U_i\end{bmatrix}\succeq 0,\quad 
\begin{bmatrix}\tilde v_i&v_i\\v_i&V_i\end{bmatrix}\succeq 0
	\end{aligned}
\end{equation*}
for $i=1,\dots,n$. Let $\lambda\in\R$ denote the dual variables of the first inequality constraints and $s,t,\mu\in\R^n$ denote the dual variable for $u\geq 0$, $v\geq 0$ and $\tilde u + \tilde v = 1$, respectively. The Lagrangian dual is given by:
 \begin{equation*}
 	\begin{aligned}
 		\max_{\lambda,s,t,\mu}\ & -\frac{1}{2}\cdot\left(\lambda(\rho^2-\|\hat x\|_2^2)+\mu^T\1\right)\\
 		\text{s.t. }\  \ & \begin{bmatrix}\mu_i&c_i-g_i-\lambda\hat x_i-s_i\\c_i-g_i-\lambda\hat x_i-s_i&\lambda\end{bmatrix}\succeq 0,\\
 		& \begin{bmatrix}\mu_i&g_i+\lambda\hat x_i-t_i\\g_i+\lambda\hat x_i-t_i&\lambda\end{bmatrix}\succeq 0,\\
         &\lambda\geq 0,\quad s\geq 0,\quad t\geq 0,\quad \mu\geq 0,
 	\end{aligned}
 \end{equation*}
 for $i=1,\dots,n$. For a $2\times 2$ matrix $X$, note that $X\succeq 0$ holds if and only if $\det(X)\ge 0$ and $\diag(X)\ge 0$. Applying this insight yields a second-order cone programming (SOCP) problem
\begin{equation}\label{eq:socp}
	\begin{aligned}
		\max_{\lambda,s,t,\mu}\ & -\frac{1}{2}\cdot\left(\lambda(\rho^2-\|\hat x\|_2^2)+\mu^T\1\right)\\
		\text{s.t. }\  \ &\lambda \mu_i\geq (c_i-g_i-s_i-\lambda\hat x_i)^2,\\
		&\lambda \mu_i\geq (g_i-t_i+\lambda\hat x_i)^2,\\
        &\lambda\geq 0,\quad s\geq 0,\quad t\geq 0,\quad \mu\geq 0,
	\end{aligned}
\end{equation}
for $i=1,\dots,n$. Due to space constraints, we defer the detailed derivation for the dual problem (\ref{eq:socp}) to the Appendix~\ref{app:dual}. We are now ready to prove Theorem~\ref{thm:sdp_crown_lb}. 
\begin{proof}
   Given any $c,g\in\R^n$. Let $\hat\rho = \rho^2-\|\hat x\|_2^2$, $a_i = c_i - g_i - \lambda\hat x_i$ and $b_i = g_i+\lambda\hat x_i$. Fixing any $\lambda\geq 0$ and optimizing $\mu$ in (\ref{eq:socp}) yields
\begin{align*}
	&\max_{\lambda,s,t\geq 0}\ -\frac{1}{2}\cdot\left(\lambda\hat\rho+\sum_{i=1}^{n}\frac{\max\left\{(a_i-s_i)^2,(b_i-t_i)^2\right\}}{\lambda}\right)\\
	=&\max_{\lambda\geq 0}\ -\frac{1}{2}\cdot\left(\lambda\hat\rho+\sum_{i=1}^{n}\frac{\min\left\{a_i,b_i,0\right\}^2}{\lambda}\right)\\
	=&\max_{\lambda\geq 0}\ h(g,\lambda)
\end{align*}
where the first equality follows from $\min_{s_i\geq 0} (a_i-s_i)^2=\min\{a_i,0\}^2$ and $\min_{t_i\geq 0} (b_i-t_i)^2=\min\{b_i,0\}^2$, and $\max\{\min\{a_i,0\}^2,\min\{b_i,0\}^2\}=\min\{a_i,b_i,0\}^2$ for any $a_i,b_i\in\R$. Since $h(g,\lambda)$ is a lower bound on (\ref{eq:hopt}) for any $\lambda\geq 0$, we have $c^T\relu(x)\geq g^Tx+h(g,\lambda)\ \forall x\in\BB_{2}(\hat{x},\rho)$ for any $g\in\R^n$, $\lambda\geq 0$.
\end{proof}
\begin{figure*}[t!]
    \centering
    \includegraphics[width=0.5\textwidth]{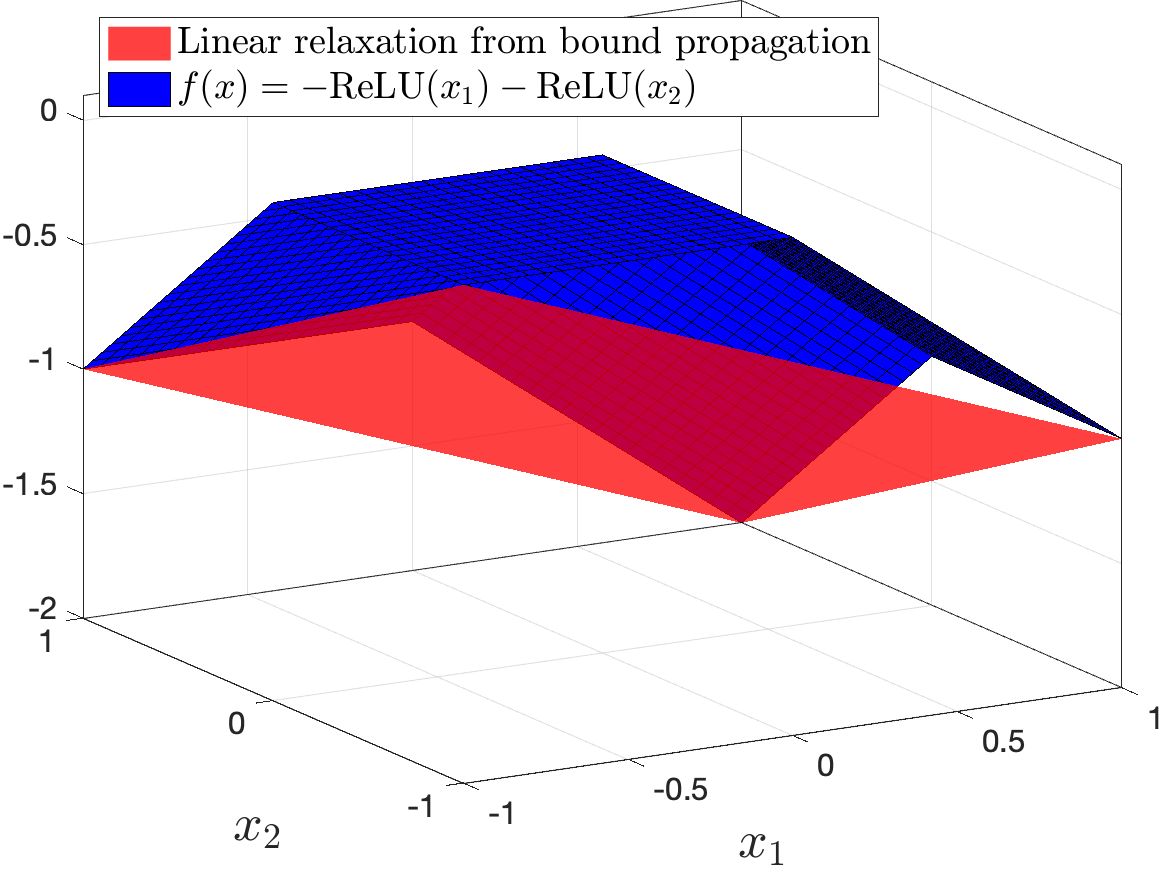}%
    \includegraphics[width=0.5\textwidth]{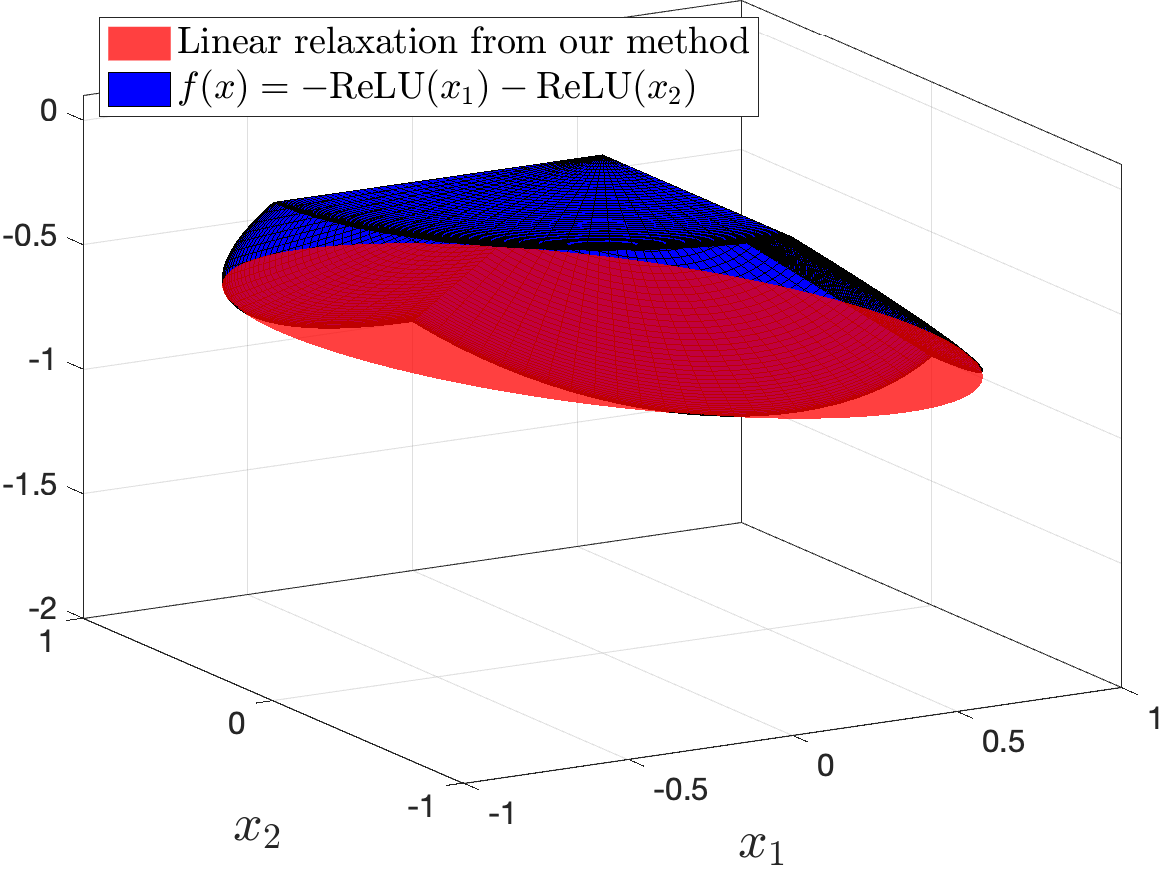}
    \caption{\textbf{Comparing linear relaxations within $\ell_2$-norm ball from our method and bound propagation.} Consider a task of finding a linear relaxation of $f(x)=-\relu(x_1)-\relu(x_2)$ on $\BB_2(0,1)$. \textbf{(Left.)} Bound propagation finds the tightest possible linear relaxation on $\BB_\infty(0,1)$, however, such relaxation is not the tightest on $\BB_2(0,1)$. \textbf{(Right.)} Our method finds the tightest possible linear relaxation on $\BB_2(0,1)$, which is tighter than bound propagation by a factor of $\sqrt{2}$.
    }\label{fig:our_lb}
    \vspace{-1em}
\end{figure*}
\section{Tightness analysis}\label{sec:tightness}
In this section, we provide theoretical  analysis on the tightness of our SDP relaxation in the special case where $f(x)\equiv\relu(x)$ and $\hat x = 0$. We show that our SDP relaxation (\ref{eq:socp}) is exactly tight in this case and guarantees at most a factor of $\sqrt{n}$ improvement over bound propagation when computing linear relaxation within $\BB_2(0,\rho)$. We begin by characterizing the linear relaxation from bound propagation $x\to g(\alpha)^Tx+h(\alpha)$ for the box $x\in\BB_\infty(0,\rho\1)\supseteq \BB_2(0,\rho)$ below.  

\begin{lemma}\label{lem:LiRPA_zero_center}
Given $c\in\R^n$ and $\rho>0$, the bound 
\begin{align*}
	g(\alpha)&=\frac{1}{2}\min\{c,0\}+\alpha\odot\max\{c,0\},\\
	h(\alpha)&=-\rho\|\min\{g(\alpha),0\}\|_1
\end{align*}
satisfies
\[
    c^T\relu(x)\geq g(\alpha)^Tx+h(\alpha)\ \ \forall\ x\in\BB_\infty(0,\rho\1)
\]
for any $0\leq\alpha\leq 1$.
\end{lemma}
\begin{proof}
    Notice that $\alpha_ix_i\leq \relu(x_i)\leq \frac{1}{2}x_i+\frac{\rho}{2}$ for any $0\leq\alpha_i\leq 1$. Therefore, we have $g(\alpha)=\frac{1}{2}\min\{c,0\}+\alpha\odot\max\{c,0\}$ and $h(\alpha)=\rho\sum_{i=1}^n\frac{1}{2}\min\{c_i,0\}=-\rho\sum_{i=1}^n|\min\{g_i(\alpha),0\}|=-\rho\|\min\{g(\alpha),0\}\|_1$.
\end{proof} 

In the following Theorem, we show that our method yields $h(g(\alpha),\lambda)=-\rho\|\min\{g(\alpha),0\}\|_2$ when $\lambda\geq 0$ is chosen optimally. On the other hand, the linear relaxation from bound propagation yields $h(\alpha)=-\rho\|\min\{g(\alpha),0\}\|_1\leq h(g(\alpha),\lambda)$ as in Lemma~\ref{lem:LiRPA_zero_center}, which is looser than our method by at most a factor of $\sqrt{n}$.

\begin{theorem}\label{thm:zero_center}
	Given $c\in\R^n$ and $\rho>0$. The following holds
    \[
    c^T\relu(x)\geq g^Tx-\rho\|\min\{c-g,g,0\}\|_2\
    \]  
    for all $x\in\BB_2(0,\rho)$ for any $g\in\R^n$.
\end{theorem}
\begin{proof}
	Setting $\hat x=0$ in $h(g,\lambda)$ and optimizing over $\lambda\geq 0$ yields
	\begin{align*}
		&\max_{\lambda\geq 0}\ -\frac{1}{2}\cdot\left(\lambda\rho^2+\frac{1}{\lambda}\|\min\left\{c-g,g,0\right\}\|_2^2\right)\\
		=&-\rho\|\min\left\{c-g,g,0\right\}\|_2
	\end{align*}
    where the equality follows from $2\sqrt{ab}=\min_{x\geq 0}ax+b/x$ for any $a,b\geq 0$.
\end{proof}
We obtain $h(g(\alpha),\lambda)=-\rho\|\min\{g(\alpha),0\}\|_2$ by substituting $g=g(\alpha)$ in Theorem~\ref{thm:zero_center}. Notice that $\min\{g(\alpha),0\} = \min\{c-g(\alpha),g(\alpha),0\}$ from Lemma~\ref{lem:LiRPA_zero_center}.  Theorem~\ref{thm:zero_center} guarantees at most a factor of $\sqrt{n}$ improvement over bound propagation when $\hat x= 0$, which we provide a simple illustration in Figure~\ref{fig:our_lb}.

Finally, we show that our SDP relaxation (\ref{eq:socp}) is exactly tight in the following Theorem, i.e., both (\ref{eq:socp}) and (\ref{eq:hopt}) attain the same optimal value. Therefore, the choice of $h(g(\alpha),\lambda)$ is optimal. 

\begin{table*}[!t]
\centering
\caption{\textbf{Verified accuracy under $\ell_2$-norm perturbations.} We report the verified accuracy (\%) for 200 images. For each method, we also report the average verification time (in seconds or hours), except for LipNaive and LipSDP, where we report the total time for computing the Lipschitz constant. The upper bound on verified accuracy is estimated using projected gradient descent. A dash "-" indicates the model could not be evaluated due to excessive computational time.}
\label{tab:verified_accuracy}
\resizebox{\textwidth}{!}{ 
\begin{threeparttable}
\begin{tabular}{l|c|c|c|c|c|c|c|c@{\hskip 5pt}c|c|c}
\toprule
  & Upper & \textbf{SDP-CROWN} & GCP-CROWN & BICCOS & $\beta$-CROWN & $\alpha$-CROWN & LipNaive & \multicolumn{2}{c|}{LipSDP} & LP-All & BM-Full\\
  & Bound & (Ours) &  &  &  & & & (split=2) & (no split) & &\\
\midrule
\multicolumn{12}{c}{\textbf{MNIST Model}$^\dagger$} \\
\midrule
MLP  & 54\% & 32.5\% (2.5s)& 41\% (173s) & 38\% (198s) & 36\% (302s)&1.5\% (1.2s)&29\% (0.02s)&29.5\% (19s)&30.5\% (62s)&9\% (75s)&\textbf{53\%} (0.3h)\\[2pt]
ConvSmall & 84.5\%&\textbf{81.5\%} (12s)&19.5\% (248s)&17\% (265s)&16\% (257s)&0\% (17s)&77.5\% (0.1s)& 78\% (875s)& 78.5\% (0.9h)& 10\% (0.6h)& -\\[2pt]
ConvLarge &84\%&\textbf{79.5\%} (88s)&0\% (309s)&0\% (304s)&0\% (307s)&0\% (66s)&77\% (1s)&-&-&-&-\\
\midrule
\multicolumn{12}{c}{\textbf{CIFAR-10 Model}$^\ddagger$} \\
\midrule
CNN-A &55.5\%&\textbf{49\%} (12s)&20\% (210s)&20\% (224s)&20\% (201s)&7.5\% (3.8s)&39\% (0.2s)& 39\% (1.7h)&-&-&-\\[2pt]
CNN-B &59.5\%&\textbf{49.5\%} (16s)&3\% (290s)&3\% (302s)&3\% (193s)&0\% (8.7s)&33\% (0.3s)&-&-&-&-\\[2pt]
CNN-C &47\%&\textbf{42.5\%} (10s)&35.5\% (96s)&36\% (101s)&35.5\% (63s)&24.5\% (5.8s)&36.5\% (0.2s)& 37\% (0.5h)&38.5\% (1h)&24.5\% (0.3h)&-\\[2pt]
ConvSmall & 52.5\%&\textbf{43.5\%} (9s)&18\% (225s)&18\% (220s)&17.5\% (146s)&6\% (4.4s)&33\% (0.2s)& 33.5\% (1.2h)&-&-&-\\[2pt]
ConvDeep & 50.5\%&\textbf{46\%} (25s)&31\% (133s)&31.5\% (133s)&30.5\% (133s)&22.5\% (9.2s)&39.5\% (0.3s)& 39.5\% (1.9h)&-&-&-\\[2pt]
ConvLarge & 72.5\%&\textbf{63.5\%} (73s)&6\% (286s)&6\% (282s)&5\% (289s)&2.5\% (47s)&47.5\% (1.2s)&-&-&-&-\\
\bottomrule
\end{tabular}
\begin{tablenotes}
{\large
\item[$\dagger$] The $\ell_2$-norm perturbation is set to be $\rho=1.0$ for MLP, and $\rho=0.3$ for both ConvSmall and ConvLarge.
\item[$\ddagger$] The $\ell_2$-norm perturbation is set to be $\rho=8/255$ for ConvLarge, and $\rho=24/255$ for all the other models.
}
\end{tablenotes}
\end{threeparttable}
}
\end{table*}

\begin{theorem}\label{thm:sdp_tightness}
	Given $c\in\R^n$ and $\rho>0$. The following holds
    \[
    -\rho\|\min\{c-g,g,0\}\|_2 = \min_{\|x\|_2\leq\rho} c^T\relu(x)-g^Tx\
    \]  
    for any $g\in\R^n$.
\end{theorem}
\begin{proof}
    Applying the positive/negative splitting $x=u-v$ where $u,v\geq 0$ and $u\odot v=0$ yields
    \[
    \min_{\substack{u,v\geq 0,\\u\odot v=0 }} \ \sum_{i=1}^{n}(c_i-g_i)u_i + g_iv_i\ \ \text{s.t.}\ \ \sum_{i=1}^{n} u_i^2 + v_i^2 \leq \rho^2.
    \]
    For each $i$, we substitute a variable $y_i$ according to the following three cases. Case 1: $c_i-g_i\leq\min\{g_i,0\}$, we have $u_i^\star\geq 0$ and $v_i^\star=0$; therefore we set $y_i=u_i$. Case 2: $g_i\leq\min\{c_i-g_i,0\}$, we have $u_i^\star=0$ and $v_i^\star\geq 0$; therefore we set $y_i=v_i$. Case 3: $0\leq\min\{c_i-g_i,g_i\}$, we have $u_i^\star=v_i^\star=0$; therefore we simply let $y_i\geq 0$. Substituting each $y_i$ yields
    \[
    \min_{\substack{y\geq 0}} \ \sum_{i=1}^{n}y_i\cdot \min\{c_i-g_i,g_i,0\}\ \ \text{s.t.}\ \ \|y\|_2 \leq \rho,
    \]
    which attains optimal value $-\rho\|\min\{c-g,g,0\}\|_2$.
\end{proof}

For the general case with $\hat x\neq 0$ and network $f(x)$ defined in (\ref{eq:model}), the improvement of our method cannot be analyzed analytically; instead, we present empirical validation in Section~\ref{sec:empirical_tightness} to show that our method can be tighter than bound propagation in general settings.



\section{Experiments}\label{sec:experiment}
In this section, we compare the practical performance of our proposed method against several state-of-the-art neural network verifiers for certifying $\ell_2$ adversaries. The source code of our proposed method is available at \mbox{\url{https://github.com/Hong-Ming/SDP-CROWN}}.

\vspace{-0.2em}
\paragraph{Methods.}
\mbox{SDP-CROWN} denotes our proposed method. The implementation details for SDP-CROWN can be found in Appendix~\ref{app:bound_propagation}. We compare SDP-CROWN against the following verifiers based on bound propagation: \mbox{$\alpha$-CROWN} \cite{xu2021fast}, a bound propagation verifier with gradient-optimized bound propagation; \mbox{$\beta$-CROWN} \cite{wang2021beta}, a verifier based on $\alpha$-CROWN that can additionally handle split constraints for ReLU neurons; \mbox{GCP-CROWN} \cite{zhang2022general} and \mbox{BICCOS} \cite{zhou2024scalable}, verifiers based on $\beta$-CROWN that can additionally handle general cutting plane constraints. Since GCP-CROWN and BICCOS use mixed-integer programming (MIP) solvers to find cutting planes, we add the $\ell_2$-norm constraint into the MIP formulation of (\ref{eq:attack}), so that all cutting planes generated from the MIP will consider the $\ell_2$-norm constraint rather than the enclosing $\ell_\infty$-norm constraint. We defer the detailed hyperparameter settings for bound propagation methods to the Appendix~\ref{app:experiment}.

We also compare SDP-CROWN against the following verifiers based on estimating an upper bound on the global Lipschitz constant: \mbox{LipNaive}, a verifier estimates the Lipschitz upper bound using the Lipschitz constant of each layer (as in Section 3 of \citet{gouk2021regularisation}); and \mbox{LipSDP} \cite{fazlyab2019efficient}, a verifier estimates the Lipschitz upper bound based on solving SDP relaxations of the Lipschitz constant estimation problem. Specifically, LipNaive lower bounds $c^Tf(x)$ within $\BB_2(\hat x,\rho)$ through $c^Tf(x) \geq c^Tf(\hat x) - \rho\cdot\|c^TW^{(N)}\|_2\cdot\|W^{(N-1)}\|_2\cdots \|W^{(1)}\|_2$ where $\|W\|_2$ denotes the spectral $\ell_2$-norm of a matrix $W$.

Finally, we compare SDP-CROWN against the following verifiers based on directly solving convex relaxations of the verification problem (\ref{eq:attack}): \mbox{LP-All} \cite{salman2019convex}, a verifier solves an LP relaxation of (\ref{eq:attack}) that uses the tightest possible preactivation bounds found by recursively solving LP problems for each preactivation; and \mbox{BM-Full} \cite{chiu2023tight}, a verifier solves an SDP relaxation of (\ref{eq:attack}) that uses the same preactivation bounds in LP-All. For LP-All and BM-Full, we use the $\ell_2$-norm as the input constraint, as in $x\in\BB_2(\hat x,\rho)$. The complexity of BM-Full and LP-All is cubic with respect to the number of preactivations, and therefore they are not scalable to most of the models used in our experiment.

\begin{figure*}[!t]
    \vspace{1em}
    \centering
    \begin{subfigure}{0.333\textwidth}
      \centering
      \includegraphics[width=\linewidth]{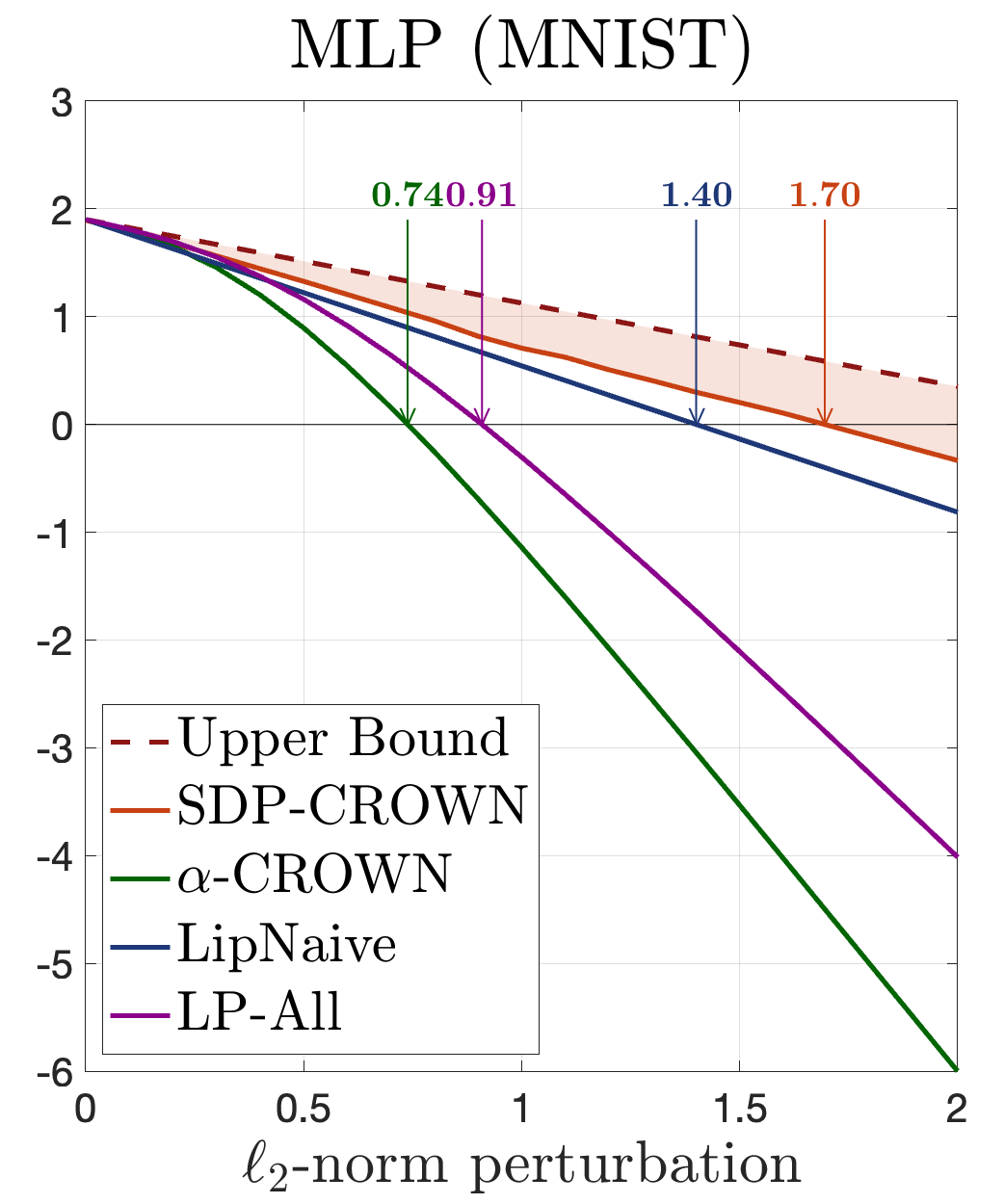}
    \end{subfigure}%
    \begin{subfigure}{0.333\textwidth}
      \centering
      \includegraphics[width=\linewidth]{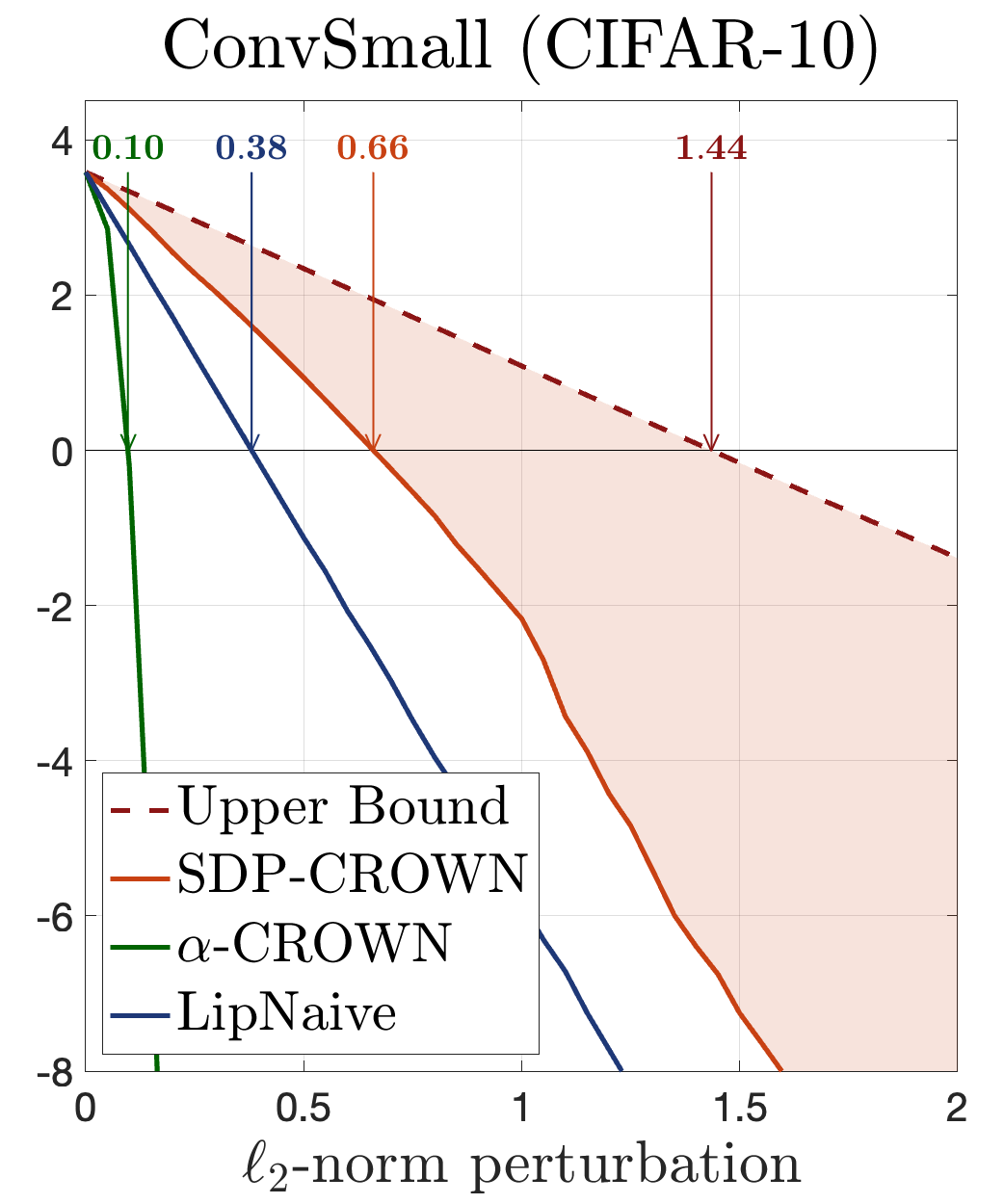}
    \end{subfigure}%
    \begin{subfigure}{0.333\textwidth}
      \centering
      \includegraphics[width=\linewidth]{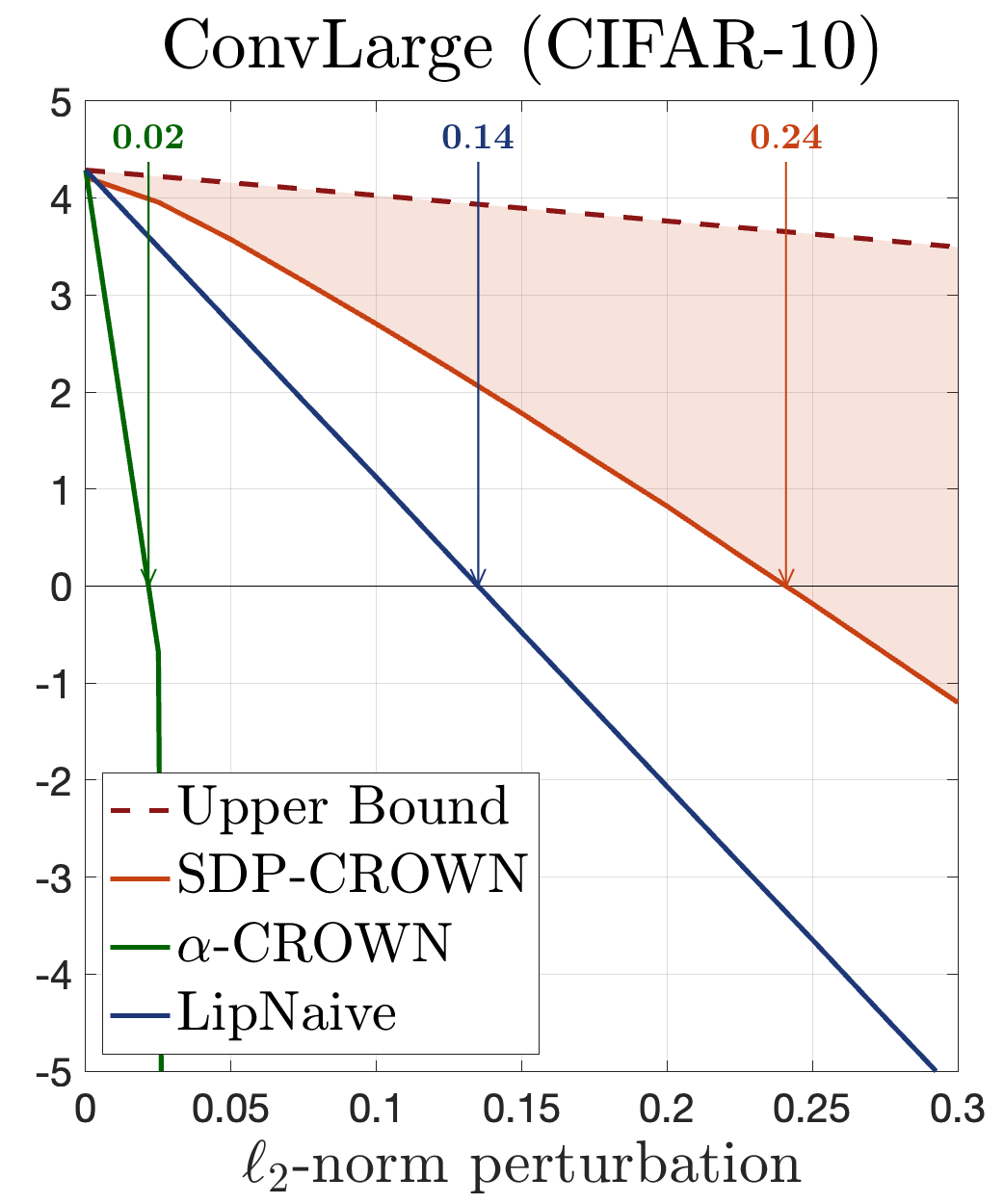}
    \end{subfigure}
    \caption{\textbf{Lower bounds on the robustness margin under $\ell_2$-norm perturbations}. We compare the lower bounds on (\ref{eq:attack}) computed from SDP-CROWN, $\alpha$-CROWN, LipNaive and LP-All. The lower bounds are averages over 90 instances of (\ref{eq:attack}). The upper bound on (\ref{eq:attack}) is estimated projected gradient descent (PGD). The numbers in the figure indicate the $\ell_2$-norm perturbation level at which each lower bound crosses zero. Note that robustness verification is only meaningful in the interval where the PGD upper bound remains positive. \textbf{(Left.)} Small-scale model MLP (MNIST). \textbf{(Middle.)} Medium-scale model ConvSmall (CIFAR-10). \textbf{(Right.)} Large-scale model ConvLarge (CIFAR-10).}\label{fig:tightness}
\end{figure*}

\vspace{-0.2em}
\paragraph{Setups.}
All the experiments are run on a machine with a single Tesla V100-SXM2 GPU (32GB GPU memory) and dual Intel Xeon Gold 6138 CPUs.

\vspace{-0.2em}
\paragraph{Models.}
All the model architectures used in our experiment are taken from \citet{wang2021beta} and \citet{leino21gloro}. To ensure non-vacuous $\ell_2$-norm robustness and make meaningful comparisons across different verification methods, we retrain all models to have a small global Lipschitz upper bound while keeping their model architecture unchanged. We defer the detailed model architecture and the training procedure to the Appendix~\ref{app:experiment}.

\subsection{Robustness verification for neural networks}\label{sec:verified accuracy}
We compare the verified accuracy of SDP-CROWN against state-of-the-art verifiers on models trained on the MNIST and CIFAR-10 datasets. In each case, we fix the $\ell_2$-norm perturbation $\rho$ and compute verified accuracy using the first 200 images in the test set. Here, the verified accuracy denotes the percentage of inputs that are both correctly classified and robust. For comparison, we also compute the upper bound on verified accuracy using projected gradient descent (PGD) attacks \cite{madry2017towards}.

Table~\ref{tab:verified_accuracy} shows the verified accuracy and the average computation time for SDP-CROWN, GCP-CROWN, BICCOS, $\beta$-CROWN, $\alpha$-CROWN, LipNaive, LipSDP, LP-All and BM-Full. While BM-Full achieves the best verified accuracy in the first case, it unfortunately becomes computationally prohibitive in all remaining cases as its complexity scales cubically with respect to the number of preactivations. In all the remaining cases, SDP-CROWN achieves verified accuracy close to the PGD upper bound while the bound propagation method $\alpha$-CROWN exhibits limited certification performance. Other bound propagation methods $\beta$-CROWN, GCP-CROWN and BICCOS can greatly improve over $\alpha$-CROWN, but the gap is still large compared to SDP-CROWN, LipNaive and LipSDP. 

Notably, SDP-CROWN is consistently tighter than LipNaive and LipSDP, suggesting that verifying robustness solely through the Lipschitz constant can be overly conservative, even for networks trained to have a small Lipschitz constant. For LipSDP, networks are divided into subnetworks when evaluating the Lipschitz constant: split=1 indicates single-layer subnetworks, split=2 denotes two-layer subnetworks, and no split means the full network is evaluated directly. We ignore reporting split=1 for LipSDP as it only marginally improves over LipNaive. 

Finally, LP-All is not scalable to large models and is also noticeably weaker than us on certified accuracy.

\subsection{Tightness of lower bounds and verified accuracy}
As the neural network verification problem (\ref{eq:attack}) is NP-hard, all methods based on finding a lower bound on (\ref{eq:attack}) via any sort of convex relaxations must become loose for sufficiently large $\ell_2$ perturbations. However, robustness verification is only necessary in the interval where the upper bound of (\ref{eq:attack}) is positive, which can be efficiently estimated via PGD. Therefore, it is crucial for the lower bound to be tight within this region.
 
In this experiment, we examine the gap between the lower bound computed from SDP-CROWN and the upper bound computed from PGD across a wide range of $\ell_2$ perturbations. To ensure an accurate evaluation, we compute the average lower bounds over 90 instances of (\ref{eq:attack}), which are generated via 9 incorrect classes for the first 10 correctly classified test images. We compare our average lower bound to $\alpha$-CROWN, LipNaive and LP-All, which are verifiers also based on solving convex relaxations of (\ref{eq:attack}). Figure~\ref{fig:tightness} reports the average lower bounds and the average PGD upper bound with respect to three models of different scales: a small-scale model MLP (MNIST); a medium-scale model ConvSmall (CIFAR-10); and a large-scale model ConvLarge (CIFAR-10).

As shown in Figure~\ref{fig:tightness}, our proposed SDP-CROWN consistently outperforms $\alpha$-CROWN, LipNaive and LP-All, and significantly narrows the gap between the PGD upper bound. Notice that $\alpha$-CROWN produces extremely loose lower bounds in almost all cases, especially for large networks such as ConvLarge (CIFAR-10), where its lower bound drops rapidly. LP-All marginally improves over $\alpha$-CROWN. We note that LP-All is excluded in ConvSmall (CIFAR-10) and ConvLarge (CIFAR-10) due to its high computational cost. LipNaive provides tighter lower bounds than $\alpha$-CROWN and LP-All as all three models are trained to have small global Lipschitz upper bounds, but is consistently looser than SDP-CROWN. 

\begin{figure}[t!]
    \vspace{1em}
    \centering
    \includegraphics[width=\linewidth]{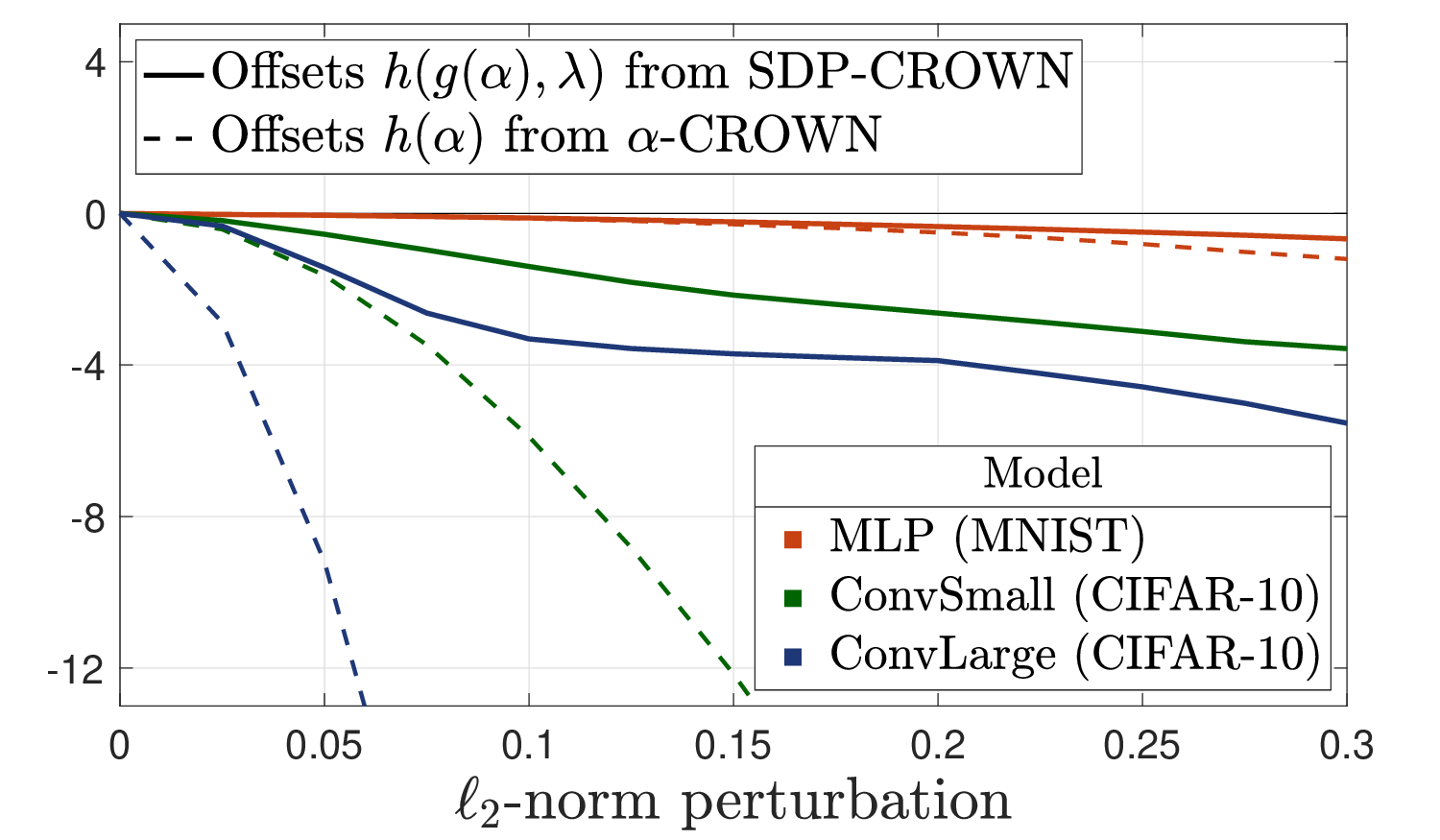}
    \caption{\textbf{Linear relaxation offsets under $\ell_2$-norm perturbations.} We compare the offset $h(\alpha)$ from $\alpha$-CROWN to the offset $h(g(\alpha),\lambda)$ from SDP-CROWN. The offsets are averages over 90 instances of (\ref{eq:attack}). \textbf{(Red line.)} Small-scale model MLP (MNIST). \textbf{(Green line.)} Medium-scale model ConvSmall (CIFAR-10). \textbf{(Blue line.)} Large-scale model ConvLarge (CIFAR-10).
    }
    \label{fig:offsets}
\end{figure}

\subsection{Tightening bound propagation}\label{sec:empirical_tightness}
In Section~\ref{sec:tightness}, we prove that when $f(x)\equiv\relu(x)$ and the center of the $\ell_2$ norm perturbation is zero, the offset $h(g(\alpha),\lambda)$ computed from SDP-CROWN is guaranteed to be at least a factor of $\sqrt{n}$ tighter than the offset $h(\alpha)$ from bound propagation methods. However, the amount of improvement cannot be analyzed analytically under general settings. To empirically demonstrate how much improvement SDP-CROWN can achieve, in this experiment, we compute the average offset $h(g(\alpha),\lambda)$ from SDP-CROWN, and the average offset $h(\alpha)$ from $\alpha$-CROWN. Specifically, the average offset of SDP-CROWN is taken over $h^{(k)}(g(\alpha^{(k)}),\lambda^{(k)})$ in (\ref{eq:sdp_crown_offset}) for $k=1,\ldots,N$, as in $\frac{1}{N}\sum_{k=1}^N h^{(k)}(g(\alpha^{(k)}),\lambda^{(k)})$, and the average offset of  $\alpha$-CROWN is taken over $h^{(k)}(\alpha^{(k)})$ in (\ref{eq:lirpa_offset})for $k=1,\ldots,N$, as in $\frac{1}{N}\sum_{k=1}^N h^{(k)}(\alpha^{(k)})$. To ensure an accurate evaluation, we compute the average offsets over 90 instances of (\ref{eq:attack}), which are generated via 9 incorrect classes for the first 10 correctly classified test images. Figure~\ref{fig:offsets} reports the average offsets with respect to three models of different scales: a small-scale model MLP (MNIST); a medium-scale model ConvSmall (CIFAR-10); and a large-scale model ConvLarge (CIFAR-10). 

As shown in Figure~\ref{fig:offsets}, SDP-CROWN consistently yields a tighter offset compared to $\alpha$-CROWN under general settings, which demonstrates its effectiveness in tightening bound propagation and improves certification quality. Notably, $\alpha$-CROWN exhibits a significant drop in $h(\alpha)$ for larger networks such as ConvLarge (CIFAR-10) and ConvSmall (CIFAR-10) as the $\ell_2$ perturbation increases. In contrast, SDP-CROWN does not experience a rapid reduction in $h(g(\alpha),\lambda)$, maintaining significantly larger offsets across all models and perturbation sizes.

\section{Conclusion}
In this work, we present SDP-CROWN, a novel framework that significantly tightens bound propagation for neural network verification under $\ell_2$-norm perturbations. SDP-CROWN leverages semidefinite programming relaxations to improve the tightness of bound propagation while retaining the efficiency of bound propagation methods. Theoretically, we prove that SDP-CROWN can be up to $\sqrt{n}$ tighter than bound propagation for a one-neuron network under zero-centered $\ell_2$-norm perturbations. Practically, our extensive experiments demonstrate that SDP-CROWN consistently outperforms state-of-the-art verifiers across a range of models under $\ell_2$-norm perturbations, including models with over 2 million parameters and 65,000 neurons, where traditional LP and SDP methods are computationally infeasible, and bound propagation methods yield notably loose relaxations. Our results establish SDP-CROWN as both a theoretical and practical advancement in scalable neural network verification under $\ell_2$-norm perturbations.

\section*{Acknowledgments}
Financial support for this work was provided by NSF CAREER Award ECCS-2047462, IIS-2331967, and ONR Award N00014-24-1-2671. Huan Zhang is supported in part by the AI2050 program at Schmidt Sciences (AI2050 Early Career Fellowship).

\section*{Impact Statement}
The work presented in this paper aims to advance neural network verification in machine learning. There are many potential societal consequences of our work, none of which we feel must be specifically highlighted here.

\bibliography{robustness}

\begin{thebibliography}{44}
\providecommand{\natexlab}[1]{#1}
\providecommand{\url}[1]{\texttt{#1}}
\expandafter\ifx\csname urlstyle\endcsname\relax
  \providecommand{\doi}[1]{doi: #1}\else
  \providecommand{\doi}{doi: \begingroup \urlstyle{rm}\Url}\fi

\bibitem[Anderson et~al.(2021)Anderson, Ma, Li, and Sojoudi]{anderson2021partition}
Anderson, B.~G., Ma, Z., Li, J., and Sojoudi, S.
\newblock Partition-based convex relaxations for certifying the robustness of relu neural networks.
\newblock \emph{arXiv preprint arXiv:2101.09306}, 2021.

\bibitem[Araujo et~al.(2023)Araujo, Havens, Delattre, Allauzen, and Hu]{araujo2023unified}
Araujo, A., Havens, A.~J., Delattre, B., Allauzen, A., and Hu, B.
\newblock A unified algebraic perspective on lipschitz neural networks.
\newblock In \emph{ICLR}, 2023.

\bibitem[Barrett et~al.(2022)Barrett, Camuto, Willetts, and Rainforth]{barrett2022certifiably}
Barrett, B., Camuto, A., Willetts, M., and Rainforth, T.
\newblock Certifiably robust variational autoencoders.
\newblock In \emph{International Conference on Artificial Intelligence and Statistics}, pp.\  3663--3683. PMLR, 2022.

\bibitem[Batten et~al.(2021)Batten, Kouvaros, Lomuscio, and Zheng]{batten2021efficient}
Batten, B., Kouvaros, P., Lomuscio, A., and Zheng, Y.
\newblock Efficient neural network verification via layer-based semidefinite relaxations and linear cuts.
\newblock In \emph{IJCAI}, pp.\  2184--2190, 2021.

\bibitem[Brix et~al.(2023)Brix, M{\"u}ller, Bak, Johnson, and Liu]{brix2023first}
Brix, C., M{\"u}ller, M.~N., Bak, S., Johnson, T.~T., and Liu, C.
\newblock First three years of the international verification of neural networks competition (vnn-comp).
\newblock \emph{International Journal on Software Tools for Technology Transfer}, 25\penalty0 (3):\penalty0 329--339, 2023.

\bibitem[Brix et~al.(2024)Brix, Bak, Johnson, and Wu]{brix2024fifth}
Brix, C., Bak, S., Johnson, T.~T., and Wu, H.
\newblock The fifth international verification of neural networks competition (vnn-comp 2024): Summary and results.
\newblock \emph{arXiv preprint arXiv:2412.19985}, 2024.

\bibitem[Brown et~al.(2022)Brown, Schmerling, Azizan, and Pavone]{brown2022unified}
Brown, R.~A., Schmerling, E., Azizan, N., and Pavone, M.
\newblock A unified view of sdp-based neural network verification through completely positive programming.
\newblock In \emph{International conference on artificial intelligence and statistics}, pp.\  9334--9355. PMLR, 2022.

\bibitem[Chiu \& Zhang(2023)Chiu and Zhang]{chiu2023tight}
Chiu, H.-M. and Zhang, R.~Y.
\newblock Tight certification of adversarially trained neural networks via nonconvex low-rank semidefinite relaxations.
\newblock In \emph{International Conference on Machine Learning}, pp.\  5631--5660. PMLR, 2023.

\bibitem[Dathathri et~al.(2020)Dathathri, Dvijotham, Kurakin, Raghunathan, Uesato, Bunel, Shankar, Steinhardt, Goodfellow, Liang, et~al.]{dathathri2020enabling}
Dathathri, S., Dvijotham, K., Kurakin, A., Raghunathan, A., Uesato, J., Bunel, R., Shankar, S., Steinhardt, J., Goodfellow, I., Liang, P., et~al.
\newblock Enabling certification of verification-agnostic networks via memory-efficient semidefinite programming.
\newblock In \emph{Advances in Neural Information Processing Systems}, 2020.

\bibitem[De~Palma et~al.(2021)De~Palma, Behl, Bunel, Torr, and Kumar]{de2021scaling}
De~Palma, A., Behl, H.~S., Bunel, R., Torr, P., and Kumar, M.~P.
\newblock Scaling the convex barrier with active sets.
\newblock In \emph{Proceedings of the ICLR 2021 Conference}. Open Review, 2021.

\bibitem[Dvijotham et~al.(2018)Dvijotham, Stanforth, Gowal, Mann, and Kohli]{dvijotham2018dual}
Dvijotham, K., Stanforth, R., Gowal, S., Mann, T.~A., and Kohli, P.
\newblock A dual approach to scalable verification of deep networks.
\newblock In \emph{UAI}, volume~1, pp.\ ~2, 2018.

\bibitem[Fazlyab et~al.(2019)Fazlyab, Robey, Hassani, Morari, and Pappas]{fazlyab2019efficient}
Fazlyab, M., Robey, A., Hassani, H., Morari, M., and Pappas, G.
\newblock Efficient and accurate estimation of lipschitz constants for deep neural networks.
\newblock \emph{Advances in neural information processing systems}, 32, 2019.

\bibitem[Fazlyab et~al.(2020)Fazlyab, Morari, and Pappas]{fazlyab2020safety}
Fazlyab, M., Morari, M., and Pappas, G.~J.
\newblock Safety verification and robustness analysis of neural networks via quadratic constraints and semidefinite programming.
\newblock \emph{IEEE Transactions on Automatic Control}, 67\penalty0 (1):\penalty0 1--15, 2020.

\bibitem[Ferrari et~al.(2022)Ferrari, Mueller, Jovanovi{\'c}, and Vechev]{ferrari2022complete}
Ferrari, C., Mueller, M.~N., Jovanovi{\'c}, N., and Vechev, M.
\newblock Complete verification via multi-neuron relaxation guided branch-and-bound.
\newblock In \emph{International Conference on Learning Representations}, 2022.

\bibitem[Gouk et~al.(2021)Gouk, Frank, Pfahringer, and Cree]{gouk2021regularisation}
Gouk, H., Frank, E., Pfahringer, B., and Cree, M.~J.
\newblock Regularisation of neural networks by enforcing lipschitz continuity.
\newblock \emph{Machine Learning}, 110:\penalty0 393--416, 2021.

\bibitem[Gowal et~al.(2019)Gowal, Dvijotham, Stanforth, Bunel, Qin, Uesato, Arandjelovic, Mann, and Kohli]{gowal2019scalable}
Gowal, S., Dvijotham, K.~D., Stanforth, R., Bunel, R., Qin, C., Uesato, J., Arandjelovic, R., Mann, T., and Kohli, P.
\newblock Scalable verified training for provably robust image classification.
\newblock In \emph{Proceedings of the IEEE/CVF International Conference on Computer Vision}, pp.\  4842--4851, 2019.

\bibitem[Hashemi et~al.(2021)Hashemi, Kouvaros, and Lomuscio]{hashemi2021osip}
Hashemi, V., Kouvaros, P., and Lomuscio, A.
\newblock Osip: Tightened bound propagation for the verification of relu neural networks.
\newblock In \emph{International Conference on Software Engineering and Formal Methods}, pp.\  463--480. Springer, 2021.

\bibitem[Hu et~al.(2023)Hu, Zou, Wang, Leino, and Fredrikson]{hu2023unlocking}
Hu, K., Zou, A., Wang, Z., Leino, K., and Fredrikson, M.
\newblock Unlocking deterministic robustness certification on imagenet.
\newblock \emph{Advances in Neural Information Processing Systems}, 36:\penalty0 42993--43011, 2023.

\bibitem[Huang et~al.(2021)Huang, Zhang, Shi, Kolter, and Anandkumar]{huang2021training}
Huang, Y., Zhang, H., Shi, Y., Kolter, J.~Z., and Anandkumar, A.
\newblock Training certifiably robust neural networks with efficient local lipschitz bounds.
\newblock \emph{Advances in Neural Information Processing Systems}, 34:\penalty0 22745--22757, 2021.

\bibitem[Leino et~al.(2021)Leino, Wang, and Fredrikson]{leino21gloro}
Leino, K., Wang, Z., and Fredrikson, M.
\newblock Globally-robust neural networks.
\newblock In \emph{International Conference on Machine Learning (ICML)}, 2021.

\bibitem[Li et~al.(2019)Li, Haque, Anil, Lucas, Grosse, and Jacobsen]{li2019preventing}
Li, Q., Haque, S., Anil, C., Lucas, J., Grosse, R.~B., and Jacobsen, J.-H.
\newblock Preventing gradient attenuation in lipschitz constrained convolutional networks.
\newblock \emph{Advances in neural information processing systems}, 32, 2019.

\bibitem[Ma \& Sojoudi(2020)Ma and Sojoudi]{ma2020strengthened}
Ma, Z. and Sojoudi, S.
\newblock Strengthened sdp verification of neural network robustness via non-convex cuts.
\newblock \emph{arXiv preprint arXiv:2010.08603}, pp.\  715--727, 2020.

\bibitem[Madry et~al.(2018)Madry, Makelov, Schmidt, Tsipras, and Vladu]{madry2017towards}
Madry, A., Makelov, A., Schmidt, L., Tsipras, D., and Vladu, A.
\newblock Towards deep learning models resistant to adversarial attacks.
\newblock In \emph{International Conference on Learning Representations}, 2018.

\bibitem[Meunier et~al.(2022)Meunier, Delattre, Araujo, and Allauzen]{meunier2022dynamical}
Meunier, L., Delattre, B.~J., Araujo, A., and Allauzen, A.
\newblock A dynamical system perspective for lipschitz neural networks.
\newblock In \emph{International Conference on Machine Learning}, pp.\  15484--15500. PMLR, 2022.

\bibitem[Newton \& Papachristodoulou(2021)Newton and Papachristodoulou]{newton2021exploiting}
Newton, M. and Papachristodoulou, A.
\newblock Exploiting sparsity for neural network verification.
\newblock In \emph{Learning for dynamics and control}, pp.\  715--727. PMLR, 2021.

\bibitem[Raghunathan et~al.(2018)Raghunathan, Steinhardt, and Liang]{raghunathan2018semidefinite}
Raghunathan, A., Steinhardt, J., and Liang, P.~S.
\newblock Semidefinite relaxations for certifying robustness to adversarial examples.
\newblock \emph{Advances in Neural Information Processing Systems}, 31, 2018.

\bibitem[Salman et~al.(2019)Salman, Yang, Zhang, Hsieh, and Zhang]{salman2019convex}
Salman, H., Yang, G., Zhang, H., Hsieh, C.-J., and Zhang, P.
\newblock A convex relaxation barrier to tight robustness verification of neural networks.
\newblock \emph{Advances in Neural Information Processing Systems}, 32, 2019.

\bibitem[Shi et~al.(2025)Shi, Jin, Kolter, Jana, Hsieh, and Zhang]{shi2024genbab}
Shi, Z., Jin, Q., Kolter, Z., Jana, S., Hsieh, C.-J., and Zhang, H.
\newblock Neural network verification with branch-and-bound for general nonlinearities.
\newblock In \emph{International Conference on Tools and Algorithms for the Construction and Analysis of Systems}, 2025.

\bibitem[Singh et~al.(2018)Singh, Gehr, Mirman, P{\"u}schel, and Vechev]{singh2018fast}
Singh, G., Gehr, T., Mirman, M., P{\"u}schel, M., and Vechev, M.
\newblock Fast and effective robustness certification.
\newblock In \emph{Advances in Neural Information Processing Systems}, pp.\  10802--10813, 2018.

\bibitem[Singh et~al.(2019)Singh, Gehr, P{\"u}schel, and Vechev]{singh2019abstract}
Singh, G., Gehr, T., P{\"u}schel, M., and Vechev, M.
\newblock An abstract domain for certifying neural networks.
\newblock \emph{Proceedings of the ACM on Programming Languages}, 3\penalty0 (POPL):\penalty0 1--30, 2019.

\bibitem[Singla \& Feizi(2022)Singla and Feizi]{singla2022improved}
Singla, S. and Feizi, S.
\newblock Improved deterministic l2 robustness on cifar-10 and cifar-100.
\newblock In \emph{International Conference on Learning Representations (ICLR)}, 2022.

\bibitem[Szegedy et~al.(2014)Szegedy, Zaremba, Sutskever, Bruna, Erhan, Goodfellow, and Fergus]{Szegedy2014intriguing}
Szegedy, C., Zaremba, W., Sutskever, I., Bruna, J., Erhan, D., Goodfellow, I., and Fergus, R.
\newblock Intriguing properties of neural networks.
\newblock In \emph{International Conference on Learning Representations}, 2014.

\bibitem[Trockman \& Kolter(2021)Trockman and Kolter]{trockman2021orthogonalizing}
Trockman, A. and Kolter, J.~Z.
\newblock Orthogonalizing convolutional layers with the cayley transform.
\newblock In \emph{International Conference on Learning Representations}, 2021.

\bibitem[Vandenberghe \& Andersen(2015)Vandenberghe and Andersen]{vandenberghe2015chordal}
Vandenberghe, L. and Andersen, M.~S.
\newblock Chordal graphs and semidefinite optimization.
\newblock \emph{Foundations and Trends in Optimization}, 1\penalty0 (4):\penalty0 241--433, 2015.

\bibitem[Wang et~al.(2018)Wang, Pei, Whitehouse, Yang, and Jana]{wang2018efficient}
Wang, S., Pei, K., Whitehouse, J., Yang, J., and Jana, S.
\newblock Efficient formal safety analysis of neural networks.
\newblock \emph{Advances in neural information processing systems}, 31, 2018.

\bibitem[Wang et~al.(2021)Wang, Zhang, Xu, Lin, Jana, Hsieh, and Kolter]{wang2021beta}
Wang, S., Zhang, H., Xu, K., Lin, X., Jana, S., Hsieh, C.-J., and Kolter, J.~Z.
\newblock Beta-crown: Efficient bound propagation with per-neuron split constraints for neural network robustness verification.
\newblock \emph{Advances in Neural Information Processing Systems}, 34:\penalty0 29909--29921, 2021.

\bibitem[Wong \& Kolter(2021)Wong and Kolter]{wong2021learning}
Wong, E. and Kolter, J.~Z.
\newblock Learning perturbation sets for robust machine learning.
\newblock In \emph{International Conference on Learning Representations}, 2021.

\bibitem[Wong \& Kolter(2018)Wong and Kolter]{wong2018provable}
Wong, E. and Kolter, Z.
\newblock Provable defenses against adversarial examples via the convex outer adversarial polytope.
\newblock In \emph{International Conference on Machine Learning}, pp.\  5286--5295, 2018.

\bibitem[Xu et~al.(2020)Xu, Shi, Zhang, Wang, Chang, Huang, Kailkhura, Lin, and Hsieh]{xu2020automatic}
Xu, K., Shi, Z., Zhang, H., Wang, Y., Chang, K.-W., Huang, M., Kailkhura, B., Lin, X., and Hsieh, C.-J.
\newblock Automatic perturbation analysis for scalable certified robustness and beyond.
\newblock \emph{Advances in Neural Information Processing Systems}, 33:\penalty0 1129--1141, 2020.

\bibitem[Xu et~al.(2021)Xu, Zhang, Wang, Wang, Jana, Lin, and Hsieh]{xu2021fast}
Xu, K., Zhang, H., Wang, S., Wang, Y., Jana, S., Lin, X., and Hsieh, C.-J.
\newblock Fast and complete: Enabling complete neural network verification with rapid and massively parallel incomplete verifiers.
\newblock In \emph{International Conference on Learning Representation (ICLR)}, 2021.

\bibitem[Xu et~al.(2022)Xu, Li, and Li]{xu2022lot}
Xu, X., Li, L., and Li, B.
\newblock Lot: Layer-wise orthogonal training on improving l2 certified robustness.
\newblock \emph{Advances in Neural Information Processing Systems}, 35:\penalty0 18904--18915, 2022.

\bibitem[Zhang et~al.(2018)Zhang, Weng, Chen, Hsieh, and Daniel]{zhang2018crown}
Zhang, H., Weng, T.-W., Chen, P.-Y., Hsieh, C.-J., and Daniel, L.
\newblock Efficient neural network robustness certification with general activation functions.
\newblock \emph{Advances in neural information processing systems}, 31, 2018.

\bibitem[Zhang et~al.(2022)Zhang, Wang, Xu, Li, Li, Jana, Hsieh, and Kolter]{zhang2022general}
Zhang, H., Wang, S., Xu, K., Li, L., Li, B., Jana, S., Hsieh, C.-J., and Kolter, J.~Z.
\newblock General cutting planes for bound-propagation-based neural network verification.
\newblock \emph{Advances in Neural Information Processing Systems}, 2022.

\bibitem[Zhou et~al.(2024)Zhou, Brix, Hanasusanto, and Zhang]{zhou2024scalable}
Zhou, D., Brix, C., Hanasusanto, G.~A., and Zhang, H.
\newblock Scalable neural network verification with branch-and-bound inferred cutting planes.
\newblock In \emph{The Thirty-eighth Annual Conference on Neural Information Processing Systems}, 2024.

\end{thebibliography}
\bibliographystyle{icml2025}

\newpage
\appendix
\onecolumn
\section{Experimental Setup}\label{app:experiment}
\paragraph{Hyperparameter settings.}
In our method SDP-CROWN, the variables $\alpha$ and $\lambda$ are solved by the Adam optimizer with 300 iterations. The learning rate is set to 0.5 and 0.05 for $\alpha$ and $\lambda$, respectively, and is decayed with a factor of 0.98 per iteration. For $\alpha$-CROWN, the variable $\alpha$ is solved by the Adam optimizer for 300 iterations, with the learning rate set to 0.5 and the decay factor set to 0.98. For $\beta$-CROWN, GCP-CROWN, and BICCOS, the timeout threshold for branch and bound is set to 300 seconds.

\paragraph{Model architecture.}
All the model architectures are taken from \citet{wang2021beta} and \citet{leino21gloro}.
\begin{table}[!htbp]
\caption{Model architectures used in our experiments.}
\label{tab:model_structures}
\centering
\aboverulesep=0ex
\belowrulesep=0ex
\resizebox{\textwidth}{!}{ 
\begin{threeparttable}
\begin{tabularx}{\textwidth}{l|X|l|l|l}
\toprule\rule{0pt}{1.1EM}
\textbf{Model name} & \textbf{Model structure} &\textbf{Parameters} &\textbf{Neurons}&\textbf{Accuracy}\\
\vspace{-1em}&&&&\\\midrule\rule{0pt}{1.1EM}
MLP (MNIST)      & Linear(784, 100) - Linear(100, 100) - Linear(100, 10) & 89,610& 994&79\%\\
\vspace{-1em}&&&&\\\midrule\rule{0pt}{1.1EM}
ConvSmall (MNIST)      & Conv(1, 16, 4, 2, 1) - Conv(16, 32, 4, 2, 1) - Linear(1568, 100) - Linear(100, 10) & 166406& 5598&87.5\%\\
\vspace{-1em}&&&&\\\midrule\rule{0pt}{1.1EM}
ConvLarge (MNIST)        & Conv(1, 32, 3, 1, 1) - Conv(32, 32, 4, 2, 1) - Conv(32, 64, 3, 1, 1) - Conv(64, 64, 4, 2, 1) - Linear(3136, 512) - Linear(512, 512) - Linear(512, 10) & 1976162& 48858&89.5\%\\
\vspace{-1em}&&&&\\\midrule\rule{0pt}{1.1EM}
CNN-A (CIFAR-10) & Conv(3, 16, 4, 2, 1) - Conv(16, 32, 4, 2, 1) - Linear(2048, 100) - Linear(100, 10) & 214918&9326& 62\%\\
\vspace{-1em}&&&&\\\midrule\rule{0pt}{1.1EM}
CNN-B (CIFAR-10) & Conv(3, 32, 5, 2, 0) - Conv(32, 128, 4, 2, 1) - Linear(8192, 250) - Linear(250, 10) & 2118856& 15876&66\%\\
\vspace{-1em}&&&&\\\midrule\rule{0pt}{1.1EM}
CNN-C (CIFAR-10) & Conv(3, 8, 4, 2, 0) - Conv(8, 16, 4, 2, 0) - Linear(576, 128) - Linear(128, 64) - Linear(64, 10) & 85218& 5650&51\%\\
\vspace{-1em}&&&&\\\midrule\rule{0pt}{1.1EM}
ConvSmall (CIFAR-10) & Conv(3, 16, 4, 2, 0) - Conv(16, 32, 4, 2, 0) - Linear(1152, 100) - Linear(100, 10) & 125318& 7934&60.5\%\\
\vspace{-1em}&&&&\\\midrule\rule{0pt}{1.1EM}
ConvDeep (CIFAR-10) & Conv(3, 8, 4, 2, 1) - Conv(8, 8, 3, 1, 1) - Conv(8, 8, 3, 1, 1) - Conv(8, 8, 4, 2, 1) - Linear(512, 100) - Linear(100, 10) & 54902& 9838&53\%\\
\vspace{-1em}&&&&\\\midrule\rule{0pt}{1.1EM}
ConvLarge (CIFAR-10)     & Conv(3, 32, 3, 1, 1) - Conv(32, 32, 4, 2, 1) - Conv(32, 64, 3, 1, 1) - Conv(64, 64, 4, 2, 1) - Linear(4096, 512) - Linear(512, 512) - Linear(512, 10) & 2466858& 65546&74\%\\
\bottomrule
\end{tabularx}
\begin{tablenotes}
\item[] \textit{Note:} Conv(3, 16, 4, 2, 0) stands for a convolutional layer with 3 input channels, 16 output channels, a $4 \times 4$ kernel, stride 2 and padding 0. Linear(1568, 100) represents a fully connected layer with 1568 input features and 100 output features. There are no max pooling / average pooling used in these models and ReLU activations are applied between consecutive layers.
\end{tablenotes}
\end{threeparttable}
}
\end{table}

\paragraph{Training procedure.}
We retrained all the models used in our experiment to make them 1-lipschitz. The training strategy involves two phases: First, we train the model using the standard cross-entropy (CE) loss. Then, we retrain the model from scratch using a combination of KL-divergence and the spectral norm of the new model as the loss. During this phase, the outputs of the initially trained model are used as labels.

\newpage
\section{Details on integrating our method into bound propagation}\label{app:bound_propagation}
Bound propagation is an efficient framework for finding valid linear lower bounds for $c^Tf(x)$ within some input set $x\in\XX$. In this section, we first give a brief overview of performing bound propagation using the elementwise bound on preactivations. We then demonstrate how our results can be integrated into this framework to yield SDP-CROWN, an extension capable of performing bound propagation using the  $\ell_2$-norm constraint on the preactivations.

We include a concrete example in Section~\ref{sec:example} to illustrate the details of both bound propagation and SDP-CROWN.

\subsection{LiRPA: bound propagation using the elementwise bound on preactivations}
To better concept of bound propagation, we express each $z^{(k)}$ as a function of the input $x$, and redefine the neural network $f(x)$ in (\ref{eq:model}) as
\[
f(x) = z^{(N)}(x),\quad z^{(k+1)}(x)=W^{(k+1)}\relu(z^{(k)}(x)),\quad z^{(1)}(x) = W^{(1)}x \quad\mbox{for $k=1\ldots,N-1$}.
\]
The bound propagation we describe in this section is the backward mode of Linear Relaxation based Perturbation Analysis (LiRPA) in \cite{xu2020automatic} that efficiently constructs linear lower bounds for $c^Tf(x)$ within an elementwise bound relaxation of the input set $\XX$:
\begin{equation}\label{eq:lirpa}
g(\alpha)^Tx+h(\alpha)\leq c^Tf(x) \ \ \forall\  x\in\BB_\infty(\tilde x,\tilde\rho)\supseteq \XX,
\end{equation}
where $\tilde x, \tilde \rho\in\R^n$ are the radius and center of the elementwise bound.

LiRPA computes (\ref{eq:lirpa}) by backward constructing linear lower bounds with respect to each preactivation $z^{(k)}(x)$
\begin{equation}\label{eq:lirpa_bound_on_z_k}
[g^{(k)}(\alpha^{(k)})]^Tz^{(k)}(x) + h^{(k)}(\alpha^{(k)}) \leq c^Tf(x)\ \ \forall\ x\in\BB_\infty(\tilde x,\tilde\rho),
\end{equation}
i.e., starting from $k=N$ all the way down to $k=1$. Here, $0\leq\alpha^{(k)}\leq 1$ are variables of the same length as $z^{(k)}(x)$, which are defined in (\ref{eq:relu_bound}). 

For $k=N$. The linear lower bound with respect to $z^{(k)}$ is simply
\[
c^Tz^{(k)}(x) \leq c^Tf(x),
\]
and therefore, $g^{(k)}(\alpha^{(k)})=c$ and $h^{(k)}(\alpha^{(k)})=0$.

For $k = N-1,\dots,1$. LiRPA takes linear lower bounds with respect to $z^{(k+1)}(x)$ ((\ref{eq:lirpa_bound_on_z_k}) at $k+1$) to construct linear bounds with respect to $z^{(k)}(x)$ ((\ref{eq:lirpa_bound_on_z_k}) at $k$). The first step is to substitute $z^{(k+1)}(x)=W^{(k+1)}\relu(z^{(k)}(x))$ to yield
\[
[c^{(k)}]^T\relu(z^{(k)}(x)) + d^{(k)}\leq c^Tf(x)
\]
where $c^{(k)}=[W^{(k+1)}]^Tg^{(k+1)}(\alpha^{(k+1)})$ and $d^{(k)}=h^{(k+1)}(\alpha^{(k+1)})$. Since $\relu(z^{(k)}(x))$ is nonlinear, LiRPA performs linear relaxation on each $\relu(z^{(k)}_i(x))$ to propagate linear bounds from $\relu(z^{(k)}(x))$ to $z^{(k)}(x)$. In particular, LiRPA constructs linear bounds $\alpha^{(k)}_i z^{(k)}_i(x)\leq \relu(z^{(k)}_i(x))\leq \beta^{(k)}_i z^{(k)}_i(x) +\gamma^{(k)}_i$ within $|z^{(k)}_i(x)-\tilde z^{(k)}_i|\leq\tilde\rho^{(k)}_i$ where
\begin{equation}\label{eq:relu_bound}
\begin{aligned}
\begin{cases}
	\alpha^{(k)}_i=1,\quad \beta^{(k)}_i=1,\quad \gamma^{(k)}_i=0 & \text{if $\tilde z^{(k)}_i-\tilde\rho^{(k)}_i \geq 0$}\\
	\alpha^{(k)}_i=0,\quad \beta^{(k)}_i=0,\quad \gamma^{(k)}_i=0 & \text{if $\tilde z^{(k)}_i+\tilde\rho^{(k)}_i \leq 0$}\\
	\alpha^{(k)}_i=\tilde\alpha^{(k)}_i,\quad \beta^{(k)}_i=\frac{\tilde z^{(k)}_i+\tilde\rho^{(k)}_i}{2\tilde\rho^{(k)}_i},\quad \gamma^{(k)}_i=-\frac{(\tilde z^{(k)}_i+\tilde\rho^{(k)}_i)(\tilde z^{(k)}_i-\tilde\rho^{(k)}_i)}{2\tilde\rho^{(k)}_i} & \text{otherwise},\\
\end{cases}
	\end{aligned}
\end{equation}
and $0\leq\tilde\alpha^{(k)}_i\leq 1$ is a free variable that can be optimized (see Figure~\ref{fig:triangle} for illustration). Here, the elementwise bound on each preactivation $|z^{(k)}_i(x)-\tilde z^{(k)}_i|\leq\tilde\rho^{(k)}_i$ can be computed by treating $z^{(k)}_i(x)$ as the output of LiRPA, i.e., $c\equiv e_i$ and $f(x)\equiv z^{(k)}(x)$. 
\begin{figure}[h!]
    \centering
    \includegraphics[width=0.33\textwidth]{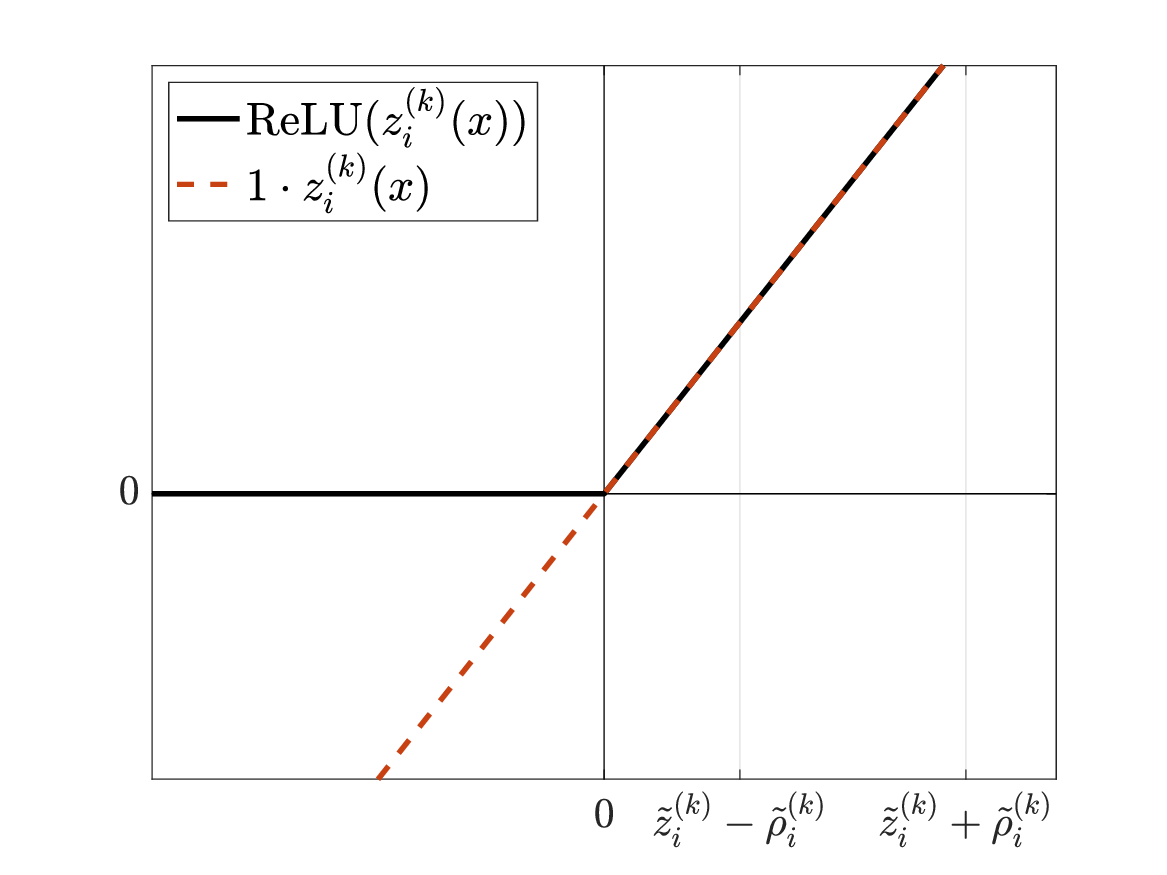}%
    \includegraphics[width=0.33\textwidth]{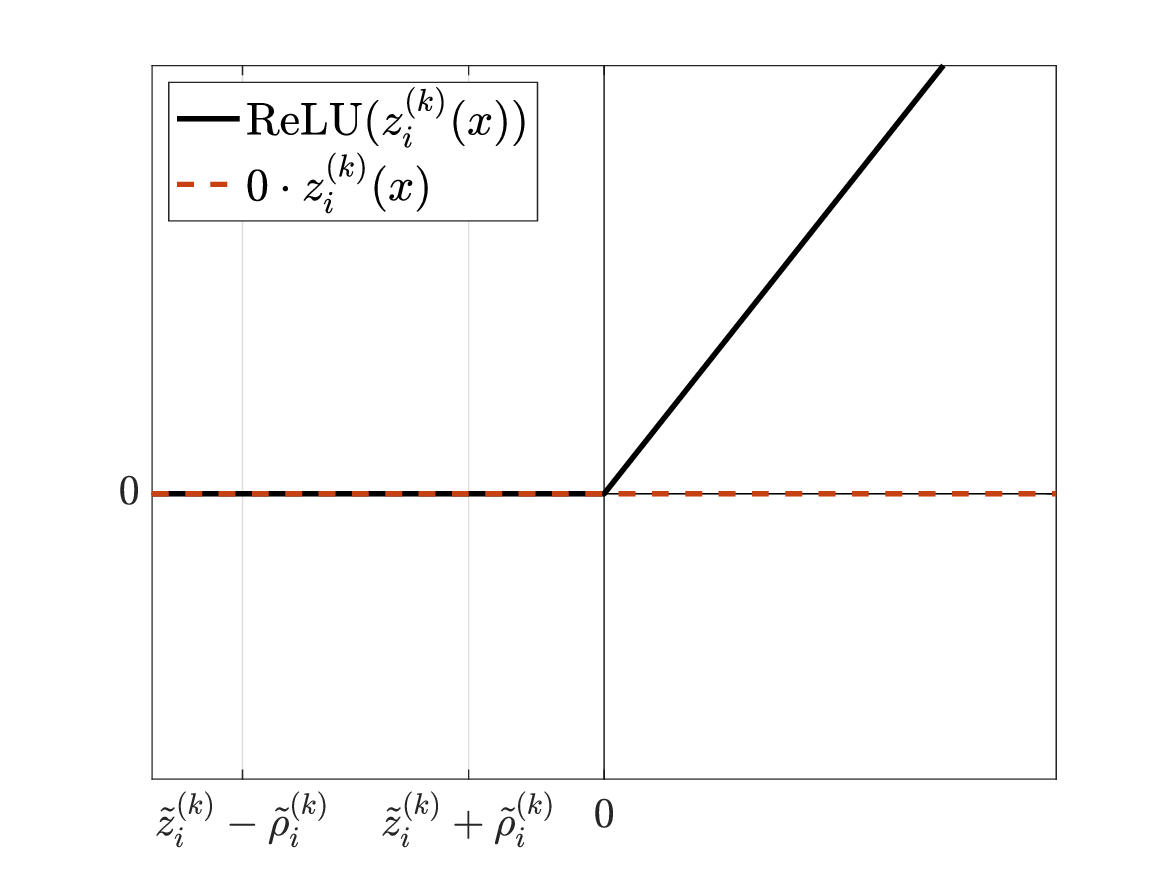}%
    \includegraphics[width=0.33\textwidth]{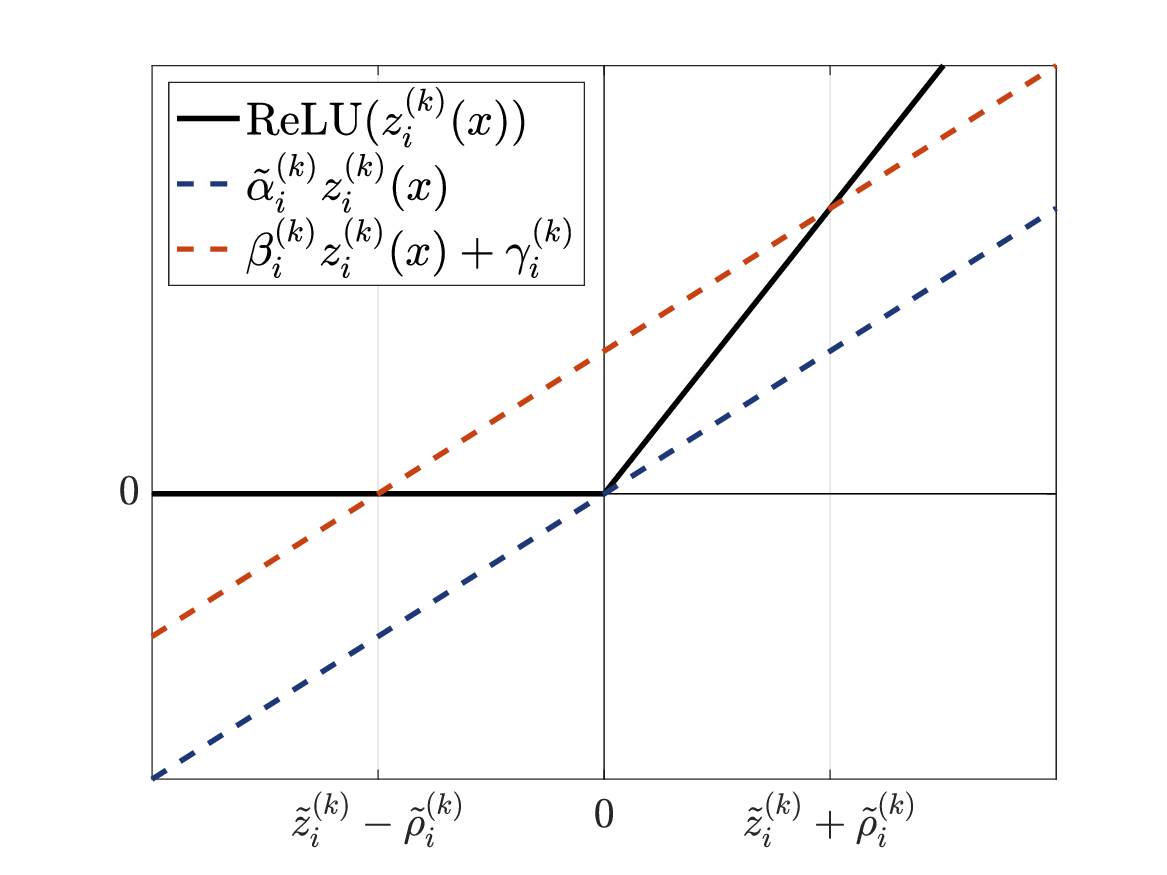}%
    \caption{\textbf{Illustration of the linear relaxation (\ref{eq:relu_bound})}. \textbf{(Left.)} $\tilde z^{(k)}_i-\tilde\rho^{(k)}_i \geq 0$. In this case, $\relu(z^{(k)}_i(x))$ is simply upper and lower bounded by $z^{(k)}_i(x)$. \textbf{(Middle.)} $\tilde z^{(k)}_i+\tilde\rho^{(k)}_i \leq 0$. In this case, $\relu(z^{(k)}_i(x))$ is simply upper and lower bounded by $0$. \textbf{(Right.)} $\tilde z^{(k)}_i-\tilde\rho^{(k)}_i\leq 0\leq \tilde z^{(k)}_i+\tilde\rho^{(k)}_i$. In this case, $\relu(z^{(k)}_i(x))$ is lower bounded by $\tilde\alpha^{(k)}_i z^{(k)}_i(x)$ for any $0\leq\tilde\alpha^{(k)}_i\leq 1$ and upper bounded by a linear function intersects $(\tilde z^{(k)}_i-\tilde\rho^{(k)}_i,0)$ and $(\tilde z^{(k)}_i+\tilde\rho^{(k)}_i,\tilde z^{(k)}_i+\tilde\rho^{(k)}_i)$. 
    }\label{fig:triangle}
\end{figure}

Using (\ref{eq:relu_bound}), the linear lower bounds with respect to $z^{(k)}$ are given by
\begin{equation}\label{eq:lirpa_offset}
[\underbrace{c^{(k)}_+\odot\alpha^{(k)} + c^{(k)}_-\odot\beta^{(k)}}_{g^{(k)}(\alpha^{(k)})}]^Tz^{(k)}(x) + \underbrace{[c^{(k)}_-]^T\gamma^{(k)} + d^{(k)}}_{h^{(k)}(\alpha^{(k)})}\leq [c^{(k)}]^T\relu(z^{(k)}(x)) + d^{(k)} \leq c^Tf(x)
\end{equation}
where $c^{(k)}_+ = \max\{c^{(k)},0\}$ and $c^{(k)}_- = \min\{c^{(k)},0\}$. 

Finally, setting $g(\alpha)=[W^{(1)}]^Tg^{(1)}(\alpha^{(1)})$  and $h(\alpha)=h^{(1)}(\alpha^{(1)})$ yields the desired linear lower bound in (\ref{eq:lirpa}).

\subsection{SDP-CROWN: bound propagation using the $\ell_2$-norm constraint on preactivations}
SDP-CROWN efficiently constructs linear lower bounds for $c^Tf(x)$ within an $\ell_2$-norm ball relaxation of the input set $\XX$:
\begin{equation}\label{eq:lirpa_sdp}
g(\alpha)^Tx+h(g(\alpha),\lambda)\leq c^Tf(x) \ \ \forall\  x\in\BB_2(\hat x,\rho)\supseteq \XX,
\end{equation}
where $\hat x\in\R^n$ and $\rho\in\R$ are the center and radius of the $\ell_2$-norm ball. To extend LiRPA to compute (\ref{eq:lirpa_sdp}), we simply set $\tilde x = \hat x$, $\tilde\rho = \rho\1$ in (\ref{eq:lirpa}) and follow the same process of LiRPA, except with the offsets $h^{(k)}(\alpha^{(k)})$ in (\ref{eq:lirpa_offset}) replaced by
\[
h^{(k)}(g^{(k)}(\alpha^{(k)}),\lambda^{(k)})=-\frac{1}{2}\cdot\left(\lambda^{(k)}\left((\rho^{(k)})^2-\|\hat z^{(k)}\|_2^2\right)+\frac{1}{\lambda^{(k)}}\|\phi^{(k)}(g^{(k)}(\alpha^{(k)}),\lambda^{(k)})\| _2^2\right) + d^{(k)}
\]
where $\lambda^{(k)}\geq 0$ is a free variable that can be optimized and 
\[
\phi_i^{(k)}(g^{(k)}(\alpha^{(k)}),\lambda^{(k)}) = \min\{c^{(k)}_i-g^{(k)}_i(\alpha^{(k)})-\lambda^{(k)}\hat z^{(k)}_i,\ g^{(k)}_i(\alpha^{(k)})+\lambda^{(k)}\hat z^{(k)}_i,\ 0\}\quad\mbox{for all $\ i$}.
\] 
Here $\hat z^{(k)}$ and $\rho^{(k)}$ are the center and radius of the $\ell_2$ norm ball for $z^{(k)}(x)$, i.e., $\|z^{(k)}(x)-\hat z^{(k)}\|_2\leq \rho^{(k)}$. The $\ell_2$ norm ball for $z^{(k)}(x)$ can be computed via the spectral norm of the weight matrices, as in $\hat z^{(k)} = W^{(k)}W^{(k-1)}\cdots W^{(1)}\hat x$ and $\rho^{(k)} = \|W^{(k)}\|_2\|W^{(k-1)}\|_2\cdots \|W^{(1)}\|_2\rho$, or by more sophisticated methods such as \cite{fazlyab2019efficient}. Here, we use $\|W^{(k)}\|_2$ to denote the spectral $\ell_2$-norm of the matrix $W^{(k)}$. We note that by Theorem~\ref{thm:sdp_crown_lb}, we always have
\begin{equation}\label{eq:sdp_crown_offset}
[g^{(k)}(\alpha^{(k)})]^Tz^{(k)}(x) + h^{(k)}(g^{(k)}(\alpha^{(k)}),\lambda^{(k)}) \leq c^Tf(x)\ \ \forall\ x\in\BB_2(\hat x,\rho)
\end{equation}
for all $k=N,\ldots,1$. 

Finally, setting $g(\alpha)=[W^{(1)}]^Tg^{(1)}(\alpha^{(1)})$ and $h(g(\alpha),\lambda)=h^{(1)}(g^{(1)}(\alpha^{(1)}),\lambda^{(1)})$ yields the desired linear lower bounds in (\ref{eq:lirpa_sdp}).

\subsection{A small example of LiRPA and SDP-CROWN}\label{sec:example}
We give a step-by-step illustration of how to find a linear lower bound on $c^Tf(x)$ within $\|x-\hat x\|_2\leq \rho$. Here, we consider a 3-layer neural network $f(x)$ with 
\[
c = 1,\quad W^{(3)}=\begin{bmatrix}-1&-1\end{bmatrix},\quad W^{(2)}=\begin{bmatrix}-1&1\\1&-1\end{bmatrix},\quad W^{(1)}=\begin{bmatrix}0&1\\1&0\end{bmatrix},\quad \hat x = \begin{bmatrix}1\\1\end{bmatrix},\quad \rho=1.
\]

\paragraph{LiRPA.} For simplicity, we compute each intermediate bound $|z^{(k)}(x)-
\tilde z^{(k)}|\leq \tilde\rho^{(k)}$ via interval bound propagation \cite{gowal2019scalable}, and always pick $\tilde\alpha^{(k)}=0$ in (\ref{eq:relu_bound}) for $k=1,2$. In this particular example, the choice of $\tilde{\alpha}^{(k)}$ does not affect the final result. The intermediate bounds are given by
\[
\tilde z^{(1)} = \begin{bmatrix}1\\1\end{bmatrix},\quad \tilde \rho^{(1)} = \begin{bmatrix}1\\1\end{bmatrix},\quad \tilde z^{(2)} = \begin{bmatrix}0\\0\end{bmatrix},\quad \tilde \rho^{(2)} = \begin{bmatrix}2\\2\end{bmatrix}.
\]

\begin{itemize}
\item Starting at $k=3$, we simply have
\[
g^{(3)}(\alpha^{(3)}) = 1,\quad h^{(3)}(\alpha^{(3)}) = 0.
\]
\item At $k=2$, substituting $z^{(3)}(x) = W^{(3)}\relu(z^{(2)}(x))$ and constructing $\alpha^{(2)}_i z^{(2)}_i(x)\leq \relu(z^{(2)}_i(x))\leq \beta^{(2)}_i z^{(2)}_i(x) +\gamma^{(2)}_i$ via $|z^{(2)}(x)-\tilde z^{(2)}|\leq \tilde\rho^{(2)}$ gives
\[
c^{(2)} = [W^{(3)}]^Tg^{(3)}(\alpha^{(3)})=\begin{bmatrix}-1\\-1\end{bmatrix},\quad d^{(2)} = h^{(3)}(\alpha^{(3)}) = 0,\quad \alpha^{(2)}=\begin{bmatrix}0\\0\end{bmatrix},\quad \beta^{(2)}=\begin{bmatrix}0.5\\0.5\end{bmatrix},\quad \gamma^{(2)}=\begin{bmatrix}1\\1\end{bmatrix}.
\]
Therefore, we have 
\[
g^{(2)}(\alpha^{(2)}) = c^{(2)}_+\odot\alpha^{(2)} + c^{(2)}_-\odot\beta^{(2)} = \begin{bmatrix}-0.5\\-0.5\end{bmatrix},\quad h^{(2)}(\alpha^{(2)}) = [c^{(2)}_-]^T\gamma^{(2)} + d^{(2)} = -2.
\]
\item At $k=1$, substituting $z^{(2)}(x) = W^{(2)}\relu(z^{(1)}(x))$ and constructing $\alpha^{(1)}_i z^{(1)}_i(x)\leq \relu(z^{(1)}_i(x))\leq \beta^{(1)}_i z^{(1)}_i(x) +\gamma^{(1)}_i$ via $|z^{(1)}(x)-\tilde z^{(1)}|\leq \tilde\rho^{(1)}$ gives
\[
c^{(1)} = [W^{(2)}]^Tg^{(2)}(\alpha^{(2)})=\begin{bmatrix}0\\0\end{bmatrix},\quad d^{(1)} = h^{(2)}(\alpha^{(2)}) = -2,\quad \alpha^{(1)}=\begin{bmatrix}1\\1\end{bmatrix},\quad \beta^{(1)}=\begin{bmatrix}1\\1\end{bmatrix},\quad \gamma^{(1)}=\begin{bmatrix}0\\0\end{bmatrix}.
\]
Therefore, we have  
\[
g^{(1)}(\alpha^{(1)}) = c^{(1)}_+\odot\alpha^{(1)} + c^{(1)}_-\odot\beta^{(1)} = \begin{bmatrix}0\\0\end{bmatrix},\quad h^{(1)}(\alpha^{(1)}) = [c^{(1)}_-]^T\gamma^{(1)} + d^{(1)} = -2.
\]
\end{itemize}

As a result, from LiRPA, we conclude 
\[
-2 = \begin{bmatrix}0\\0\end{bmatrix}^Tx + -2 \leq 1\cdot f(x)\ \ \forall\ \left\Vert x-
\begin{bmatrix}1\\1\end{bmatrix}	\right\Vert_2\leq 1.
\]

\paragraph{SDP-CROWN.} For simplicity, we compute each intermediate bound $\|z^{(k)}(x)-
\hat z^{(k)}\|_2\leq \rho^{(k)}$ for $k=1,2$ via the Lipschitz constant of $W^{(k)}$, which is given by
\[
\hat z^{(1)} = \begin{bmatrix}1\\1\end{bmatrix},\quad \rho^{(1)} = 1,\quad \hat z^{(2)} = \begin{bmatrix}0\\0\end{bmatrix},\quad \rho^{(2)} = 2.
\]
\begin{itemize}
\item Starting at $k=3$, we simply have
\[
g^{(3)}(\alpha^{(3)}) = 1,\quad h^{(3)}(g^{(3)}(\alpha^{(3)}),\lambda^{(3)}) = 0.
\]
\item At $k=2$, taking $c^{(2)}$, $g^{(2)}(\alpha^{(2)})$ from LiRPA, setting $d^{(2)} = h^{(3)}(g^{(3)}(\alpha^{(3)}),\lambda^{(3)}) = 0$, and substituting them into 
\begin{align*}
h^{(2)}(g^{(2)}(\alpha^{(2)}),\lambda^{(2)})&=-\frac{1}{2}\cdot\left(\lambda^{(2)}\left((\rho^{(2)})^2-\|\hat z^{(2)}\|_2^2\right)+\frac{1}{\lambda^{(2)}}\|\phi^{(2)}(g^{(2)}(\alpha^{(2)}),\lambda^{(2)})\| _2^2\right) + d^{(2)}\\
&=-\frac{1}{2}\cdot\left(4\lambda^{(2)}+\frac{1}{\lambda^{(2)}}\cdot 0.5\right).
\end{align*}
Notice that $\max_{\lambda^{(2)}\geq 0} h^{(2)}(g^{(2)}(\alpha^{(2)}),\lambda^{(2)})$ is maximized at $\lambda^{(2)} = \sqrt{1/8}$, and hence we have 
\[
g^{(2)}(\alpha^{(2)}) = \begin{bmatrix}-0.5\\-0.5\end{bmatrix},\quad h^{(2)}(g^{(2)}(\alpha^{(2)}),\lambda^{(2)}) = -\sqrt{2},\quad \lambda^{(2)} = \sqrt{1/8}.
\]
\item At $k=1$, taking $c^{(1)}$, $g^{(1)}(\alpha^{(1)})$ from LiRPA, setting $d^{(1)} = h^{(2)}(g^{(2)}(\alpha^{(2)}),\lambda^{(2)}) = -\sqrt{2}$, and substituting them into 
\begin{align*}
h^{(1)}(g^{(1)}(\alpha^{(1)}),\lambda^{(1)})&=-\frac{1}{2}\cdot\left(\lambda^{(1)}\left((\rho^{(1)})^2-\|\hat z^{(1)}\|_2^2\right)+\frac{1}{\lambda^{(1)}}\|\phi^{(1)}(g^{(1)}(\alpha^{(1)}),\lambda^{(1)})\| _2^2\right) + d^{(1)}\\
&=-\frac{1}{2}\cdot\left(-\lambda^{(1)}+\frac{1}{\lambda^{(1)}}\cdot \|\min\{-\lambda^{(1)}\hat z^{(1)},\lambda^{(1)}\hat z^{(1)},0 \}\|_2^2\right) -\sqrt{2}\\
&=-\frac{1}{2}\cdot\left(-\lambda^{(1)}+\lambda^{(1)}\cdot \|\min\{-\hat z^{(1)},\hat z^{(1)},0 \}\|_2^2\right) -\sqrt{2}\\
&=-\frac{1}{2}\cdot\left(\lambda^{(1)}\right) -\sqrt{2}.
\end{align*}
Obviously, $\max_{\lambda^{(1)}\geq 0} h^{(1)}(g^{(1)}(\alpha^{(1)}),\lambda^{(1)})$ is maximized at $\lambda^{(1)} = 0$, and hence
\[
g^{(1)}(\alpha^{(1)}) = \begin{bmatrix}0\\0\end{bmatrix},\quad h^{(1)}(g^{(1)}(\alpha^{(1)}),\lambda^{(1)}) = -\sqrt{2},\quad \lambda^{(1)} = 0.
\]
\end{itemize}
As a result, from our method, we conclude 
\[
-\sqrt{2} = \begin{bmatrix}0\\0\end{bmatrix}^Tx + -\sqrt{2} \leq 1\cdot f(x)\ \ \forall\ \left\Vert x-
\begin{bmatrix}1\\1\end{bmatrix}	\right\Vert_2\leq 1.
\]
In this particular example, our method tightens bound propagation by exactly a factor of $\sqrt{2}$.

\section{Some extensions of SDP-CROWN}\label{app:extension}
In this section, we describe several extensions that can further tighten SDP-CROWN. 

\subsection{Ellipsoid constraints}
The tightness of SDP-CROWN hinges crucially on the quality of the $\ell_2$-norm ball relaxation $\BB_2(\hat z^{(k)},\rho^{(k)})\supseteq \{z^{(k)}\mid x\in \BB_2(\hat x,\rho)\}$ for the input set at each $z^{(k)}$ during the computation of linear lower bounds (\ref{eq:sdp_crown_offset}). However, in general settings, $\ell_2$-norm balls might not be the best choice to relax the input set at $z^{(k)}$. For illustration, consider a simple one-layer example with
\[
W^{(1)}=\begin{bmatrix}0.5&0.5\\1.5&-0.5\end{bmatrix},\quad \hat x=\begin{bmatrix}0\\0\end{bmatrix},\quad \rho=1.
\]
As illustrated in Figure~\ref{fig:ellipsoid}, the input set at $z^{(1)}$, $\{z^{(1)}\mid x\in \BB_2(\hat x,\rho)\}$, is a rotated and elongated ellipsoid; therefore, relaxing this input set by naively propagating $\ell_2$-norm ball from $x$ to $z^{(1)}$ can result in extremely loose relaxation. To address this issue, we generalize SDP-CROWN to handle ellipsoids of the following form
\[
\EE_2(\hat x, \hat\rho) = \{x\mid \|\diag(\hat\rho)^{-1}(x-\hat x)\|_2\leq 1\}
\]
where $\hat x,\hat\rho\in\R^n$ are the center and axes of the ellipsoid. 

We note that the ellipsoid can also be efficiently propagated via the Lipschitz constant of $W^{(k)}$. Given $\EE_2(\hat z^{(k)}, \hat\rho^{(k)})$, one simple heuristic is to select the center and axis for $\EE_2(\hat z^{(k+1)}, \hat\rho^{(k+1)})$ as
\[
\hat z^{(k+1)} = W^{(k+1)}\hat z^{(k)},\quad \hat\rho^{(k+1)}=y\cdot\|\diag(y)^{-1}W^{(k+1)}\diag(\hat\rho^{(k)})\|_2 \mbox{ with } y = \|W^{(k+1)}\diag(\hat\rho^{(k)})\|_{r,2}
\]
where $\|W\|_{r,2}=\sqrt{(W\odot W)\1}$ denotes the rowwise $\ell_2$ norm of a matrix $W$.
\begin{figure}[h]
	\centering
	\includegraphics[width=0.33\linewidth]{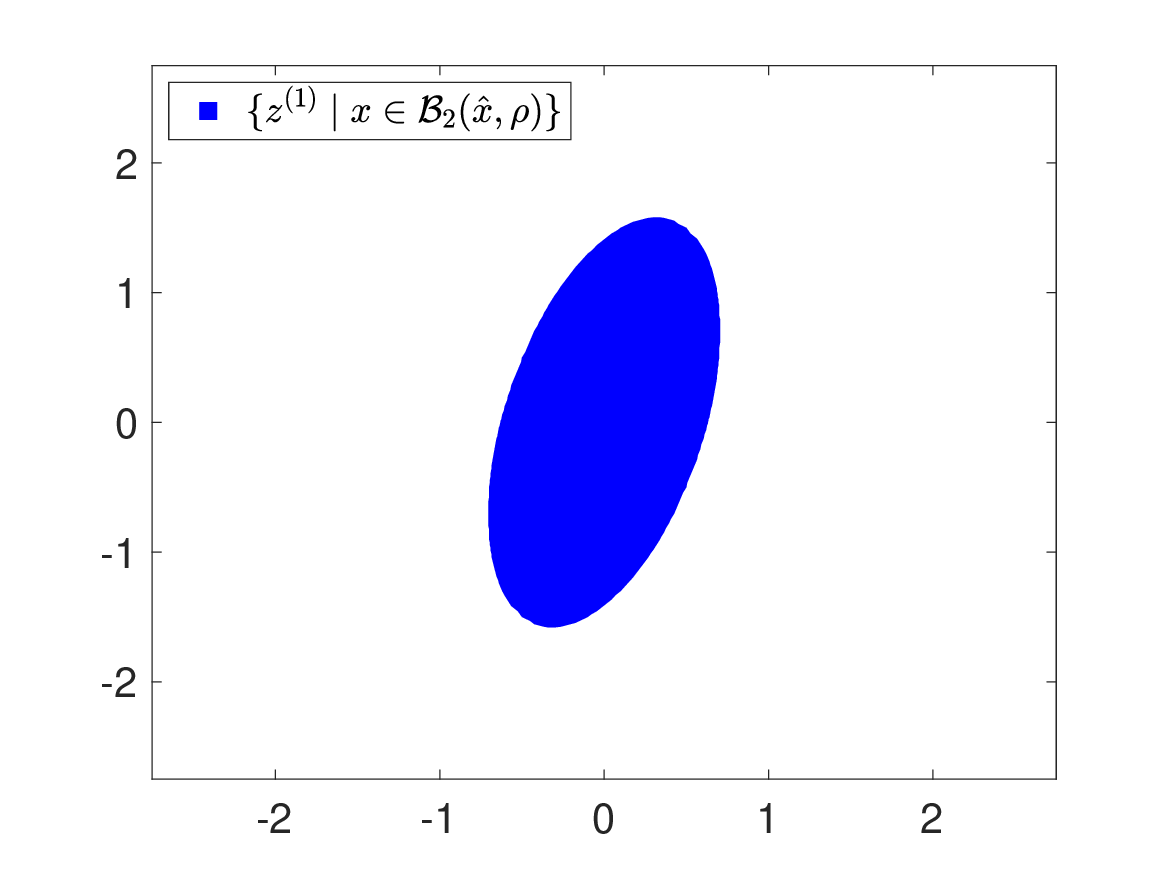}%
	\includegraphics[width=0.33\linewidth]{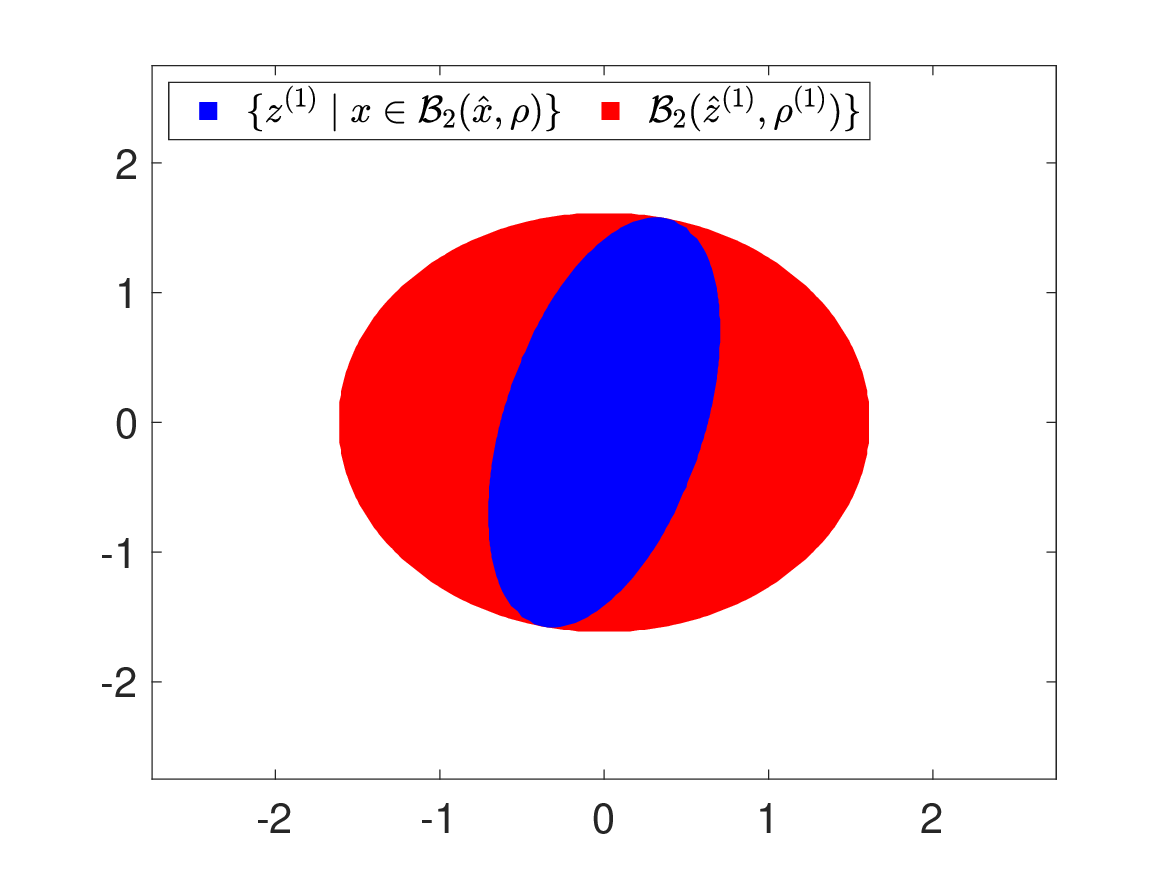}%
	\includegraphics[width=0.33\linewidth]{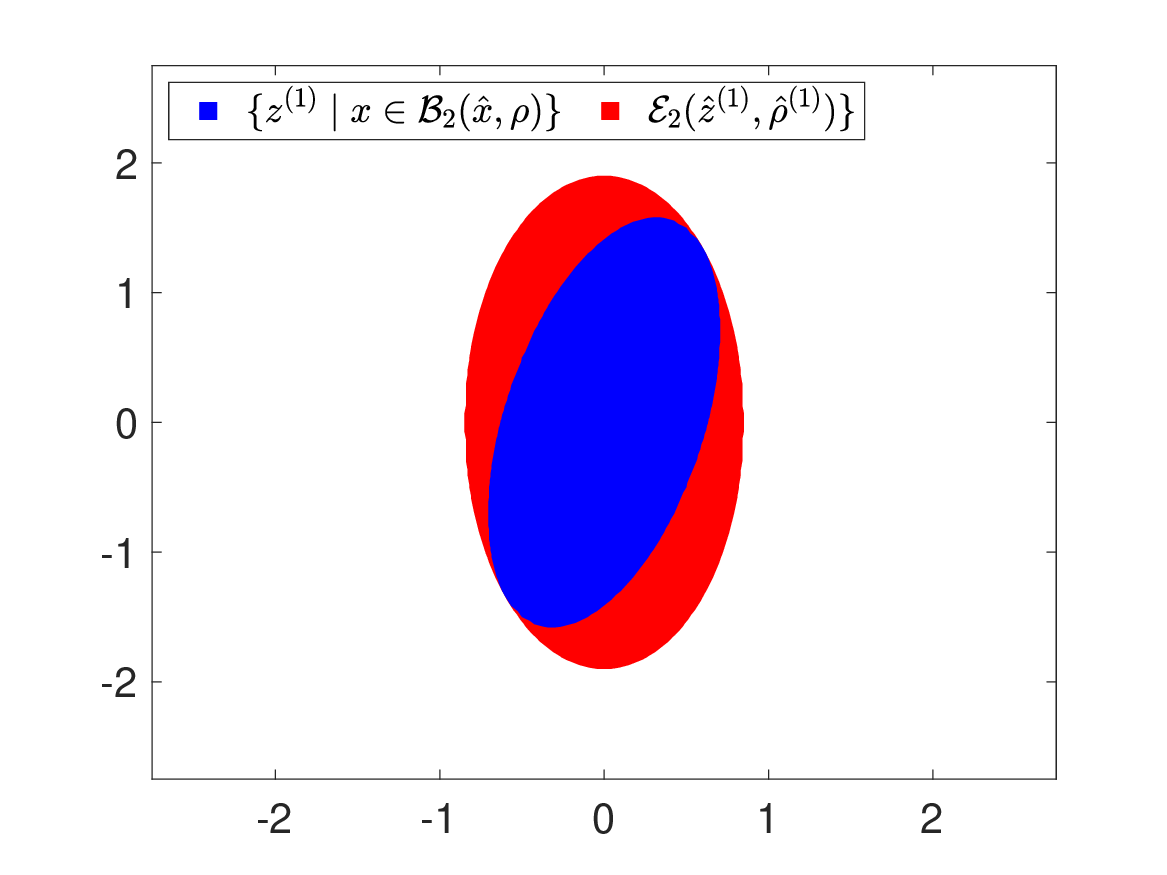}
	\caption{Constructing the $\ell_2$-norm ball and the ellipsoid relaxation at $z^{(1)}$ for a one layer neural network with $W^{(1)}=[0.5,0.5;1.5,-0.5]$, $\hat x = [0;0]$ and $\rho=1$. (\textbf{Left.}) The input set at $z^{(1)}$ with respect to the $\ell_2$-norm ball input set $\BB_2(\hat x,\rho)$ at $x$. The input set at $z^{(1)}$ is a rotated and elongated ellipsoid. (\textbf{Middle.}) The $\ell_2$-norm ball relaxation $\BB_2(\hat z^{(1)},\rho^{(1)})$ at $z^{(1)}$, where $\hat z^{(1)}=[0;0]$ and $\rho^{(1)}=\|W^{(1)}\|_2\rho=1.5302$. $\ell_2$-norm ball does not have enough degree of freedom to capture the shape of the input set at $z^{(1)}$. (\textbf{Right.}) The ellipsoid relaxation $\EE_2(\hat z_1,\hat\rho^{(1)})$ at $z^{(1)}$, where $\hat z^{(1)}=[0;0]$ and $\hat\rho^{(1)}=[0.5464;1.9360]$. The ellipsoid can better capture the shape of the input set at $z^{(1)}$.}\label{fig:ellipsoid}
\end{figure}

\subsection{Intersection between ellipsoid and elementwise constraints}
Another extension of SDP-CROWN is to also take the elementwise bound $\BB_\infty(\tilde z^{(k)}, \tilde\rho^{(k)})$ at $z^{(k)}$ into account when computing the relaxation of (\ref{eq:hopt}). In particular, SDP-CROWN can be further tightened by considering the intersection between the ellipsoid $\EE_2(\hat z^{(k)}, \hat\rho^{(k)})$ and the elementwise bound $\BB_\infty(\tilde z^{(k)}, \tilde\rho^{(k)})$ at $z^{(k)}$. For illustration, in Figure~\ref{fig:intersection}, we plot the  intersection between the elementwise bound $\BB_\infty(\tilde z^{(k)}, \tilde\rho^{(k)})$ and the ellipsoid $\EE_2(\hat z^{(k)}, \hat\rho^{(k)})$ using the same one-layer example in Figure~\ref{fig:ellipsoid}. In this case, the intersection removes four corners of $\BB_\infty(\tilde z^{(k)}, \tilde\rho^{(k)})$. We note that as the dimension increases, the number of corners removed grows exponentially.

\begin{figure}[h]
	\centering
	\includegraphics[width=0.33\linewidth]{figures/exp.eps}%
	\includegraphics[width=0.33\linewidth]{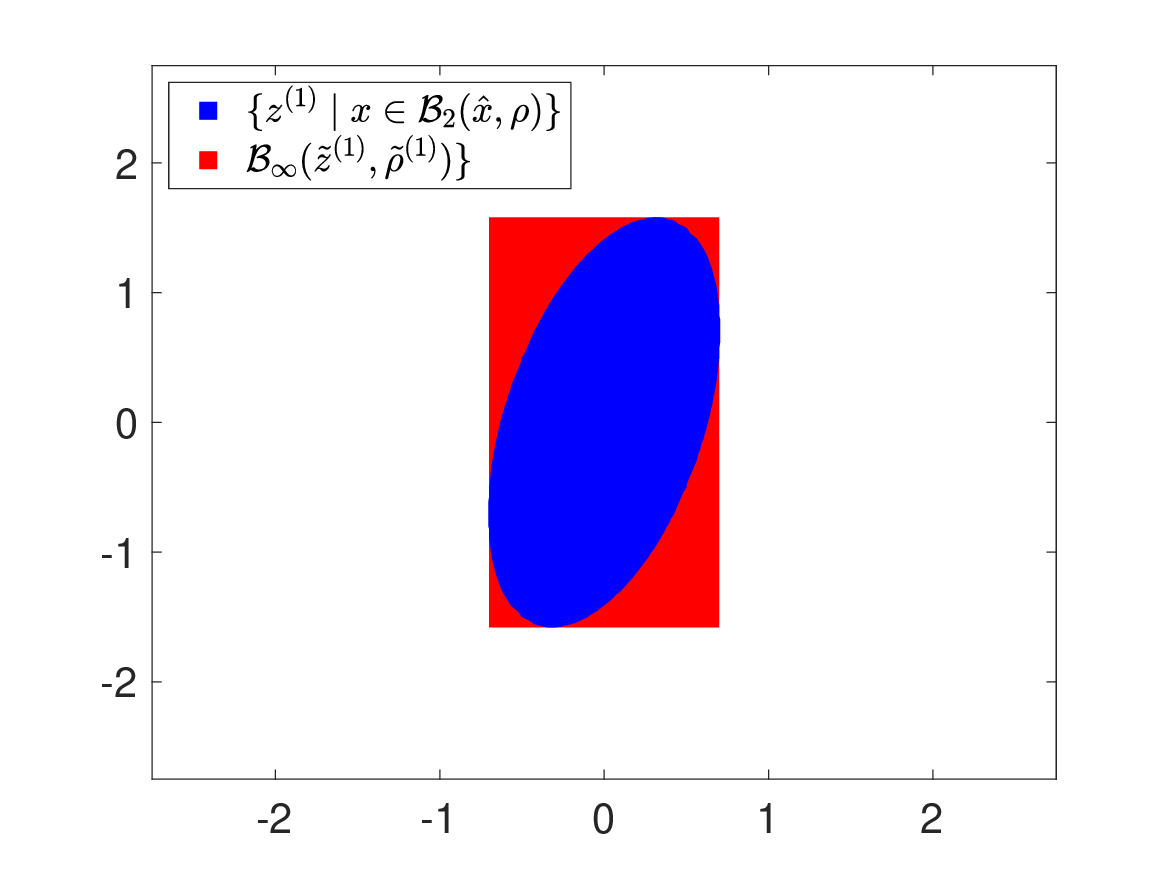}%
	\includegraphics[width=0.33\linewidth]{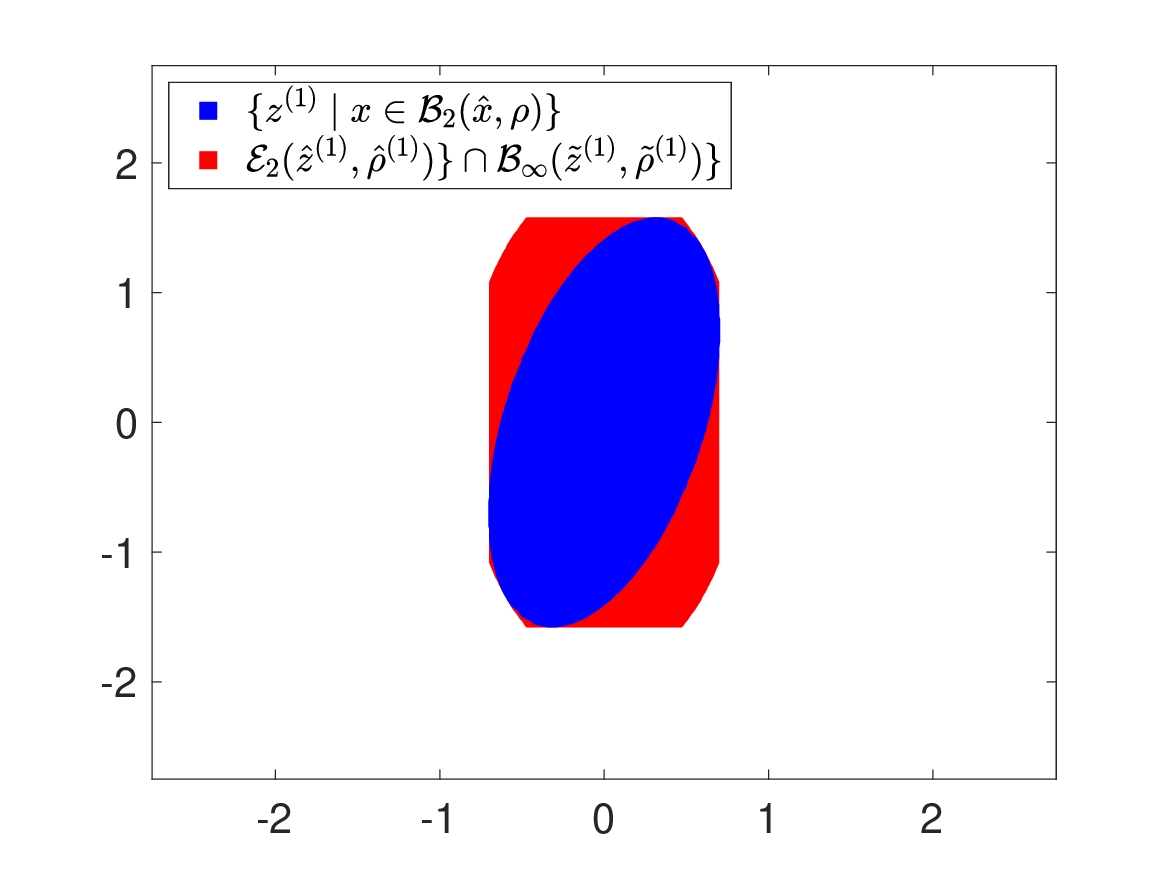}
	\caption{Constructing the intersection between ellipsoid and elementwise bound as a relaxation at $z^{(1)}$ for a one layer neural network with $W^{(1)}=[0.5,0.5;1.5,-0.5]$, $\hat x = [0;0]$ and $\rho=1$. (\textbf{Left.}) The input set at $z^{(1)}$ with respect to the $\ell_2$-norm ball input set $\BB_2(\hat x,\rho)$ at $x$. The input set at $z^{(1)}$ is a rotated and elongated ellipsoid. (\textbf{Middle.}) The elementwise bound relaxation $\BB_\infty(\tilde z^{(1)},\tilde \rho^{(1)})$ at $z^{(1)}$, where $\tilde z^{(1)}=[0;0]$ and $\tilde \rho^{(1)}=\|W^{(1)}\|_{r,2}\rho=[0.7071; 1.5811]$. (\textbf{Right.}) The intersection between ellipsoid $\EE_2(\hat z_1,\hat\rho^{(1)})$ and elementwise bound $\BB_\infty(\tilde z^{(1)},\tilde \rho^{(1)})$ at $z^{(1)}$, where $\hat z^{(1)}=[0;0]$ and $\hat\rho^{(1)}=[0.5464;1.9360]$. The intersection removes the corners of $\BB_\infty(\tilde z^{(1)},\tilde \rho^{(1)})$.}\label{fig:intersection}
\end{figure}

We can simply accommodate $\BB_\infty(\tilde z^{(k)}, \tilde\rho^{(k)})$ by adding the following inequality constraint into (\ref{eq:hopt})
\[
\relu(x_i)\leq \beta_i^{(k)}x + \gamma_i^{(k)}
\]
where $\beta_i^{(k)}$ and $\gamma_i^{(k)}$ are defined in (\ref{eq:relu_bound}). We summarized the extension of SDP-CROWN for handling the intersection between the elementwise bound $\BB_\infty(\tilde z^{(k)}, \tilde\rho^{(k)})$ and the ellipsoid $\EE_2(\hat z^{(k)}, \hat\rho^{(k)})$ in the following Theorem.

\begin{theorem}\label{thm:sdp_crown_lb_extension} 
Given $c,\hat x,\hat\rho,\tilde x,\tilde\rho\in\R^n$ where $\hat\rho,\tilde\rho\geq 0$. The following holds
\[
c^T\relu(x)\geq g^Tx+h(g,\lambda,\tau)\ \ \forall\ x\in\EE_{2}(\hat x,\hat\rho)\cap\BB_{\infty}(\tilde{x},\tilde\rho)
\]
for any $\lambda,\tau\geq 0$ and $g\in\R^n$ where 
\[
h(g,\lambda,\tau)=-\frac{1}{2}\left(\lambda(1-\|\diag(\hat\rho)^{-1}\hat x\|_2^2)+2\tau^T(\tilde\rho\odot\tilde\rho-\tilde x\odot\tilde x)+\frac{1}{\lambda}\|\phi(g,\lambda,\tau)\|_2^2\right)
\]
and 
\[
\phi_i(g,\lambda,\tau) = \hat\rho_i\cdot\min\{c_i-g_i+\tau_i(\tilde\rho_i-\tilde x_i)-\lambda\hat\rho_i^{-2}\hat x_i,g_i+\tau_i(\tilde\rho_i+\tilde x_i)+\lambda\hat\rho_i^{-2}\hat x_i,0\}.
\]
\end{theorem}

\subsection{Proof of Theorem~\ref{thm:sdp_crown_lb_extension}}
Given a linear relaxation $c^T\relu(x)\geq g^Tx+h$ that holds within $x\in\EE_{2}(\hat x,\hat\rho)\cap\BB_{\infty}(\tilde{x},\tilde\rho)$, the process of finding the tightest possible $h$ within $\EE_{2}(\hat x,\hat\rho)\cap\BB_{\infty}(\tilde{x},\tilde\rho)$ admits the following generic problem
\begin{equation*}
\begin{aligned}
\min_{x\in\R^n}\ & c^T\relu(x)-g^Tx\quad\text{s.t.}\quad \|\diag(\hat\rho)^{-1}(x-\hat x)\|_2\leq 1,\quad \relu(x_i)\leq \frac{\tilde\rho_i+\tilde x_i}{2\tilde\rho_i}x+\frac{\tilde\rho_i^2-\tilde x_i^2}{2\tilde\rho_i}\quad\text{for $i=1,\ldots,n$}.
\end{aligned}
\end{equation*}
Without loss of generality, we assume $\tilde x_i-\tilde\rho_i\leq 0\leq \tilde x_i+\tilde\rho_i$ for all $i$. Applying the positive/negative splitting $x=u-v$ where $u,v\geq 0$ and $u\odot v=0$ yields the following
\begin{equation}\label{eq:hopt_extension}
\begin{aligned}
\min_{u,v\in\R^n}\ & c^Tu-g^T(u-v)\\
\text{s.t. }\ &\sum_{i=1}^{n} (\hat\rho_i^{-1}u_i)^2 - 2\hat\rho_i^{-2}(u_i-v_i)\hat x_i + (\hat\rho_i^{-1}v_i)^2\leq 1 - \|\diag(\hat\rho)^{-1}\hat x\|_2^2,\\
& (\tilde\rho_i-\tilde x_i)u_i + (\tilde\rho_i+\tilde x_i)v_i \leq \tilde\rho_i^2-\tilde x_i^2\quad\text{for $i=1,\ldots,n$},\\
& u\geq 0,\quad v\geq 0,\quad u\odot v=0.
\end{aligned}
\end{equation}
The SDP relaxation of (\ref{eq:hopt_extension}) reads:
\begin{equation*}
	\begin{aligned}
		\min_{\tilde u,\tilde v,u,v,U,V\in\R^n}\ & c^Tu-g^T(u-v)\\
		\text{s.t. }\quad\ &\sum_{i=1}^{n} \hat\rho_i^{-2}U_i - 2\hat\rho_i^{-2}(u_i-v_i)\hat x_i + \hat\rho_i^{-2}V_i\leq 1 - \|\diag(\hat\rho)^{-1}\hat x\|_2^2,\\
& (\tilde\rho_i-\tilde x_i)u_i + (\tilde\rho_i+\tilde x_i)v_i \leq \tilde\rho_i^2-\tilde x_i^2\quad\text{for $i=1,\ldots,n$},\\
		&u\geq 0,\quad v\geq 0,\quad \tilde u + \tilde v = 1,\\
		&\begin{bmatrix}\tilde u_i&u_i\\u_i&U_i\end{bmatrix}\succeq 0,\quad 
\begin{bmatrix}\tilde v_i&v_i\\v_i&V_i\end{bmatrix}\succeq 0\quad\text{for $i=1,\ldots,n$}.
	\end{aligned}
\end{equation*}
Let $\lambda\in\R$ denote the dual variable of the first inequality constraints, $\tau_i\in\R$ denote the dual variable of each $(\tilde\rho_i-\tilde x_i)u_i + (\tilde\rho_i+\tilde x_i)v_i \leq \tilde\rho_i^2-\tilde x_i^2$, and $s,t,\mu\in\R^n$ denote the dual variable for $u\geq 0$, $v\geq 0$ and $\tilde u + \tilde v = 1$, respectively. The Lagrangian dual is given by
\begin{equation*}
	\begin{aligned}
		\max_{\lambda,\tau,s,t,\mu}\ & -\frac{1}{2}\cdot\left(\lambda(1-\|\diag(\hat\rho)^{-1}\hat x\|_2^2) + 2\tau^T(\tilde\rho\odot\tilde\rho-\tilde x\odot\tilde x) + \mu^T\1\right)\\
		\text{s.t. }\  \ & \begin{bmatrix}\mu_i&c_i-g_i+\tau_i(\tilde\rho_i-\tilde x_i)-\lambda\hat\rho_i^{-2}\hat x_i-s_i\\c_i-g_i+\tau_i(\tilde\rho_i-\tilde x_i)-\lambda\hat\rho_i^{-2}\hat x_i-s_i&\hat\rho_i^{-2}\lambda\end{bmatrix}\succeq 0\quad\text{for $i=1,\ldots,n$},\\
		& \begin{bmatrix}\mu_i&g_i+\tau_i(\tilde\rho_i+\tilde x_i)+\lambda\hat\rho_i^{-2}\hat x_i-t_i\\g_i+\tau_i(\tilde\rho_i+\tilde x_i)+\lambda\hat\rho_i^{-2}\hat x_i-t_i&\hat\rho_i^{-2}\lambda\end{bmatrix}\succeq 0\quad\text{for $i=1,\ldots,n$},\\
        &\lambda\geq 0,\quad\tau\geq 0,\quad s\geq 0,\quad t\geq 0,\quad \mu\geq 0.
	\end{aligned}
\end{equation*}
For a $2\times 2$ matrix, note that $X\succeq 0$ holds if and only if $\det(X)\ge 0$ and $\diag(X)\ge 0$. Applying this insight yields a second-order cone programming (SOCP) problem
\begin{equation}\label{eq:socp_extension}
	\begin{aligned}
		 \max_{\lambda,\tau,s,t,\mu}\ & -\frac{1}{2}\cdot\left(\lambda(1-\|\diag(\hat\rho)^{-1}\hat x\|_2^2) + 2\tau^T(\tilde\rho\odot\tilde\rho-\tilde x\odot\tilde x) + \mu^T\1\right)\\
		\text{s.t. }\  \ &\mu_i\lambda\geq \hat\rho_i^2(c_i-g_i+\tau_i(\tilde\rho_i-\tilde x_i)-\lambda\hat\rho_i^{-2}\hat x_i-s_i)^2,\\
		&\mu_i\lambda\geq \hat\rho_i^2(g_i+\tau_i(\tilde\rho_i+\tilde x_i)+\lambda\hat\rho_i^{-2}\hat x_i-t_i)^2,\\
        &\lambda\geq 0,\quad\tau\geq 0,\quad s\geq 0,\quad t\geq 0,\quad \mu\geq 0.
	\end{aligned}
\end{equation}
We are now ready to prove Theorem~\ref{thm:sdp_crown_lb_extension}. 
\begin{proof}
    Given any $c,g\in\R^n$. Let $a_i = \hat\rho_i(c_i-g_i+\tau_i(\tilde\rho_i-\tilde x_i)-\lambda\hat\rho_i^{-2}\hat x_i)$ and $b_i = \hat\rho_i(g_i+\tau_i(\tilde\rho_i+\tilde x_i)+\lambda\hat\rho_i^{-2}\hat x_i)$. Fixing any $\lambda,\tau\geq 0$ and optimizing $\mu$ in (\ref{eq:socp_extension}) yields
\begin{align*}
	&\max_{\lambda,s,t\geq 0}\ -\frac{1}{2}\cdot\left(\lambda(1-\|\diag(\hat\rho)^{-1}\hat x\|_2^2) + 2\tau^T(\tilde\rho\odot\tilde\rho-\tilde x\odot\tilde x)+\sum_{i=1}^{n}\frac{\max\left\{(a_i-s_i)^2,(b_i-t_i)^2\right\}}{\lambda}\right)\\
	=&\max_{\lambda\geq 0}\ -\frac{1}{2}\cdot\left(\lambda(1-\|\diag(\hat\rho)^{-1}\hat x\|_2^2) + 2\tau^T(\tilde\rho\odot\tilde\rho-\tilde x\odot\tilde x)+\sum_{i=1}^{n}\frac{\min\left\{a_i,b_i,0\right\}^2}{\lambda}\right)\\
	=&\max_{\lambda\geq 0}\ h(g,\lambda,\tau)
\end{align*}
where the first equality follows from $\min_{s_i\geq 0} (a_i-s_i)^2=\min\{a_i,0\}^2$ and $\min_{t_i\geq 0} (b_i-t_i)^2=\min\{b_i,0\}^2$, and $\max\{\min\{a_i,0\}^2,\min\{b_i,0\}^2\}=\min\{a_i,b_i,0\}^2$ for any $a_i,b_i\in\R$. Since $h(g,\lambda,\tau)$ is a lower bound on (\ref{eq:hopt_extension}) for any $\lambda,\tau\geq 0$, we have $c^T\relu(x)\geq g^Tx+h(g,\lambda,\tau)\ \forall x\in\EE_{2}(\hat x,\hat\rho)\cap\BB_{\infty}(\tilde{x},\tilde\rho)$ for any $g\in\R^n$, $\lambda,\tau\geq 0$.
\end{proof}

\section{Derivation of the dual problem (\ref{eq:socp})}\label{app:dual}
Recall that we have the primal problem
\begin{equation*}
	\begin{aligned}
		\min_{\tilde u,\tilde v,u,v,U,V\in\R^n}\ & c^Tu-g^T(u-v)\\
		\text{s.t. }\quad\ & (U+V)^T\1-2(u-v)^T\hat x \leq \rho^2-\|\hat x\|_2^2,\\
		&u\geq 0,\quad v\geq 0,\quad \tilde u + \tilde v = 1,\\  &\begin{bmatrix}\tilde u_i&u_i\\u_i&U_i\end{bmatrix}\succeq 0,\quad 
\begin{bmatrix}\tilde v_i&v_i\\v_i&V_i\end{bmatrix}\succeq 0\quad\text{for $i=1,\dots,n$}.
	\end{aligned}
\end{equation*}
Let $\lambda\geq 0$ denote the dual variables of the first inequality constraints. $s,t\geq 0, \mu\in\R^n$ denote the dual variable for $u\geq 0$, $v\geq 0$ and $\tilde u + \tilde v = 1$, respectively. $\begin{bmatrix}\tilde y_i& y_i\\y_i&Y_i\end{bmatrix}\succeq 0,\begin{bmatrix}\tilde z_i& z_i\\z_i&Z_i\end{bmatrix}\succeq 0$ denote the dual variables of the last two PSD constraints for $i=1,\ldots n$. The Lagrangian is given by
\begin{align*}
	&\mathcal L(\tilde u,\tilde v,u,v,U,V,\lambda,s,t,\mu,\tilde y,\tilde z,y,z,Y,Z)=\sum_{i=1}^{n}c_iu_i-g_i(u_i-v_i)\\
	&+\left[\sum_{i=1}^{n}\lambda(U_i+V_i)-2\lambda\hat x_i(u_i-v_i) \right] - \lambda(\rho^2-\|\hat x\|_2^2)\\
	&-\sum_{i=1}^{n}(s_iu_i+t_iv_i) + \sum_{i=1}^{n}\mu_i(\tilde u_i+\tilde v_i - 1)\\
	&-\sum_{i=1}^{n}\inner{\begin{bmatrix}\tilde y_i& y_i\\y_i&Y_i\end{bmatrix}}{\begin{bmatrix}\tilde u_i&u_i\\u_i&U_i\end{bmatrix}}-\sum_{i=1}^{n}\inner{\begin{bmatrix}\tilde z_i& z_i\\z_i&Z_i\end{bmatrix}}{\begin{bmatrix}\tilde v_i&v_i\\v_i&V_i\end{bmatrix}}.
\end{align*}
Rearranging the terms, we have
\begin{align}
	&\mathcal L(\tilde u,\tilde v,u,v,U,V,\lambda,s,t,\mu,\tilde y,\tilde z,y,z,Y,Z)=- \lambda(\rho^2-\|\hat x\|_2^2)-\sum_{i=1}^{n}\mu_i\nonumber \\
	&+\sum_{i=1}^{n}\inner{\begin{bmatrix}
	\mu_i-\tilde y_i& \frac{1}{2}(c_i-g_i-2\lambda\hat x_i-s_i)-y_i\\\frac{1}{2}(c_i-g_i-2\lambda\hat x_i-s_i)-y_i&\lambda-Y_i
	\end{bmatrix}}{\begin{bmatrix}\tilde u_i&u_i\\u_i&U_i\end{bmatrix}}\label{eq:dual_1} \\
	&+\sum_{i=1}^{n}\inner{\begin{bmatrix}
	\mu_i-\tilde z_i& \frac{1}{2}(c_i+2\lambda\hat x_i-t_i)-z_i\\\frac{1}{2}(c_i+2\lambda\hat x_i-t_i)-z_i&\lambda-Z_i
	\end{bmatrix}}{\begin{bmatrix}\tilde v_i&v_i\\v_i&V_i\end{bmatrix}}\label{eq:dual_2}.
\end{align}
Minimizing the Lagrangian over the primal variables yields
\begin{align*}
	&\min_{\tilde u,\tilde v,u,v,U,V\in\R^n} \mathcal L(\tilde u,\tilde v,u,v,U,V,\lambda,s,t,\mu,\tilde y,\tilde z,y,z,Y,Z)\\
	=&\begin{cases}
		- \lambda(\rho^2-\|\hat x\|_2^2)-\sum_{i=1}^{n}\mu_i & \text{if $(\ref{eq:dual_1})=0$ and $(\ref{eq:dual_2})=0\ \forall \tilde u,\tilde v,u,v,U,V\in\R^n$}\\
		-\infty &\text{otherwise}
	\end{cases}
\end{align*} 
where
\begin{align*}
	&(\ref{eq:dual_1})=0\quad\iff\quad \begin{bmatrix}\tilde y_i& y_i\\y_i&Y_i\end{bmatrix}=\begin{bmatrix}\mu_i& \frac{1}{2}(c_i-g_i-2\lambda\hat x_i-s_i)\\\frac{1}{2}(c_i-g_i-2\lambda\hat x_i-s_i)&\lambda\end{bmatrix}\quad\forall i\in\{1,\ldots,n\}\\
	&(\ref{eq:dual_2})=0\quad\iff\quad \begin{bmatrix}\tilde z_i& z_i\\z_i&Z_i\end{bmatrix}=\begin{bmatrix}\mu_i& \frac{1}{2}(c_i+2\lambda\hat x_i-t_i)\\\frac{1}{2}(c_i+2\lambda\hat x_i-t_i)&\lambda\end{bmatrix}\quad\forall i\in\{1,\ldots,n\}.
\end{align*}
Hence, the Lagrangian dual is given by
\begin{equation*}
	\begin{aligned}
		\max_{\lambda,s,t,\mu}\ & -\lambda(\rho^2-\|\hat x\|_2^2)-\mu^T\1\\
		\text{s.t. }\  \ & \begin{bmatrix}\mu_i& \frac{1}{2}(c_i-g_i-2\lambda\hat x_i-s_i)\\\frac{1}{2}(c_i-g_i-2\lambda\hat x_i-s_i)&\lambda\end{bmatrix}\succeq 0\quad\text{for $i=1,\dots,n$},\\
		& \begin{bmatrix}\mu_i& \frac{1}{2}(c_i+2\lambda\hat x_i-t_i)\\\frac{1}{2}(c_i+2\lambda\hat x_i-t_i)&\lambda\end{bmatrix}\succeq 0\quad\text{for $i=1,\dots,n$},\\
        &\lambda\geq 0,\quad s\geq 0,\quad t\geq 0,\quad \mu\geq 0.
	\end{aligned}
\end{equation*}
Rescaling $\lambda\equiv \frac{1}{2}\lambda$ and $\mu\equiv\frac{1}{2}\mu$ to yield
\begin{equation*}
	\begin{aligned}
		\max_{\lambda,s,t,\mu}\ & -\frac{1}{2}\lambda(\rho^2-\|\hat x\|_2^2)-\frac{1}{2}\mu^T\1\\
		\text{s.t. }\  \ & \begin{bmatrix}\mu_i&c_i-g_i-\lambda\hat x_i-s_i\\c_i-g_i-\lambda\hat x_i-s_i&\lambda\end{bmatrix}\succeq 0\quad\text{for $i=1,\dots,n$},\\
		& \begin{bmatrix}\mu_i&g_i+\lambda\hat x_i-t_i\\g_i+\lambda\hat x_i-t_i&\lambda\end{bmatrix}\succeq 0\quad\text{for $i=1,\dots,n$},\\
        &\lambda\geq 0,\quad s\geq 0,\quad t\geq 0,\quad \mu\geq 0.
	\end{aligned}
\end{equation*}
For a $2\times 2$ matrix, note that $X\succeq 0$ holds if and only if $\det(X)\ge 0$ and $\diag(X)\ge 0$. Finally, applying this insight yields the desired dual problem:
\begin{equation*}
	\begin{aligned}
		\frac{1}{2}\cdot \max_{\lambda,s,t,\mu}\ & -\lambda(\rho^2-\|\hat x\|_2^2)-\mu^T\1\\
		\text{s.t. }\  \ &\lambda \mu_i\geq (c_i-g_i-s_i-\lambda\hat x_i)^2\quad\text{for $i=1,\dots,n$},\\
		&\lambda \mu_i\geq (-g_i+t_i-\lambda\hat x_i)^2\quad\text{for $i=1,\dots,n$},\\
        &\lambda\geq 0,\quad s\geq 0,\quad t\geq 0,\quad \mu\geq 0.
	\end{aligned}
\end{equation*}

\end{document}